\documentclass[journal]{IEEEtran}

\usepackage[american]{babel}
\usepackage{amsmath,amssymb,amsfonts,amsthm}
\usepackage{algorithmic}
\usepackage{algorithm}
\usepackage{graphicx}
\usepackage{booktabs}
\usepackage{xcolor}
\usepackage[hidelinks]{hyperref}
\usepackage{cleveref}
\usepackage{multirow}
\usepackage{microtype}
\usepackage{enumitem}
\usepackage{graphicx}
\usepackage{dblfloatfix}
\usepackage{mathtools}
\usepackage{textcomp}
\usepackage{cite}
\usepackage{thmtools}
\usepackage{lipsum}
\usepackage{amsmath,amssymb}
\usepackage[T1]{fontenc}
\usepackage{tikz}
\usepackage{amsmath}
\usepackage{graphicx}              
\usepackage{amsmath,amssymb}                     
\usepackage{amsmath}
\usepackage{caption}
\usetikzlibrary{positioning,arrows.meta,calc,fit,backgrounds,shapes.geometric}

\definecolor{ink}{HTML}{1F2A30}
\definecolor{gmid}{HTML}{6E7B82}
\definecolor{ghair}{HTML}{C2C9CD}
\definecolor{gfill}{HTML}{EDF0F1}
\definecolor{teal}{HTML}{2E6F6A}
\definecolor{tealf}{HTML}{D8E6E4}
\definecolor{terra}{HTML}{C2622F}
\definecolor{terraf}{HTML}{F1E0D5}
\usetikzlibrary{positioning,arrows.meta,calc,fit,backgrounds,shapes.geometric}
\definecolor{tealx}{HTML}{2A6F6B}\definecolor{inkx}{HTML}{20323A}
\definecolor{fillx}{HTML}{EEF4F3}\definecolor{anat}{HTML}{E4EEF0}
\definecolor{diag}{HTML}{AFCDCA}\definecolor{rustx}{HTML}{B0461E}
\definecolor{ttint}{HTML}{EAF3F2}\definecolor{rtint}{HTML}{FBEFE9}
\usetikzlibrary{positioning,arrows.meta,calc,fit,backgrounds}
\usepackage{xcolor}
\definecolor{tealx}{HTML}{2A6F6B}
\definecolor{inkx}{HTML}{20323A}
\definecolor{fillx}{HTML}{EEF4F3}
\definecolor{anat}{HTML}{E4EEF0}
\definecolor{diag}{HTML}{AFCDCA}
\definecolor{rustx}{HTML}{B0461E}
\usetikzlibrary{positioning,arrows.meta,shapes.geometric,calc,fit,backgrounds}
\usepackage{tcolorbox}
\tcbuselibrary{skins,breakable}

\usepackage[T1]{fontenc}
\usepackage{textcomp}
\usepackage{listings}
\usepackage{xcolor}
\usepackage{upquote}          
\usepackage{listings}
\usepackage{xcolor}
\usepackage{upquote}
\tcbuselibrary{skins,listings,breakable}
\definecolor{promptbg}{gray}{0.965}
\definecolor{promptframe}{gray}{0.78}
\definecolor{prompttitle}{RGB}{15,122,122}   
\definecolor{promptarrow}{gray}{0.50}

\lstdefinestyle{promptstyle}{
  basicstyle=\footnotesize\ttfamily,
  upquote=true,
  breaklines=true,
  breakatwhitespace=true,
  breakindent=1.4em,
  postbreak=\mbox{\textcolor{promptarrow}{$\hookrightarrow$}\space},
  columns=fullflexible,
  keepspaces=true,
  showstringspaces=false,
  aboveskip=0pt, belowskip=0pt,
}

\newtcblisting{promptbox}[2][]{
  listing only,
  listing options={style=promptstyle},
  enhanced, breakable=false,
  colback=promptbg, colframe=promptframe, boxrule=0.5pt, arc=2pt,
  coltitle=white, colbacktitle=prompttitle, fonttitle=\small\bfseries,
  title={#2}, toptitle=2pt, bottomtitle=2pt,
  left=5pt, right=4pt, top=3pt, bottom=3pt,
  #1
}

\definecolor{conftriage-blue}{HTML}{1F4E79}
\definecolor{conftriage-teal}{HTML}{0E8388}
\definecolor{conftriage-orange}{HTML}{D8730A}
\definecolor{conftriage-red}{HTML}{B23A48}
\definecolor{conftriage-gray}{HTML}{4A4A4A}
\definecolor{conftriage-lightblue}{HTML}{DEEBF7}
\definecolor{conftriage-lightgreen}{HTML}{E2F0CB}
\definecolor{conftriage-lightorange}{HTML}{FFE6CC}

\theoremstyle{plain}
\newtheorem{theorem}{Theorem}

\theoremstyle{definition}

\newtheorem{assumption}{Assumption}

\theoremstyle{remark}
\newtheorem{remark}{Remark}

\newcommand{\method}{ConfTriage}
\newcommand{\xfeat}{\mathbf{x}}
\newcommand{\attr}{\mathbf{a}}
\newcommand{\ctv}{\mathbf{v}}
\newcommand{\desc}{s}
\newcommand{\imgstats}{\mathbf{g}}
\newcommand{\labelvar}{y}
\newcommand{\Pdist}{P}

\newcommand{\Dcal}{\mathcal{D}}
\newcommand{\Dcalib}{\Dcal_{\mathrm{cal}}}
\newcommand{\Dtest}{\Dcal_{\mathrm{test}}}
\newcommand{\Xcal}{\mathcal{X}}

\newcommand{\LLM}{\mathcal{M}}
\newcommand{\specdl}{f_{\mathrm{DL}}}
\newcommand{\probhat}{\hat{p}}

\newcommand{\verbconf}{c_{v}}

\newcommand{\sigmoidf}{\sigma}
\newcommand{\thresh}{\tau}
\newcommand{\plattmap}{g_{\mathrm{Platt}}}
\newcommand{\phimap}{\phi}
\newcommand{\Real}{\mathbb{R}}
\newcommand{\indicator}{\mathbb{1}}

\newcommand{\Risk}{R}

\newcommand{\Sens}{\mathrm{Sens}}

\newcommand{\defind}{D}
\newcommand{\predY}{\hat{\labelvar}}

\newcommand{\recbadge}[1]{%
  \tcbox[on line, nobeforeafter, box align=base,
    boxsep=0pt, left=4pt, right=4pt, top=1pt, bottom=1pt,
    arc=2pt, boxrule=0pt,
    colback=conftriage-teal, colframe=conftriage-teal]{%
    \color{white}\bfseries\footnotesize #1}}

\newcommand{\rec}[3]{%
  \par\noindent\recbadge{#1}\enspace\textbf{#2}\par
  \vspace{2pt}%
  \begingroup\setlength{\leftskip}{1.5em}\noindent #3\par\endgroup}

\newcommand{\recsep}{\par\vspace{5pt}%
  \textcolor{conftriage-teal!35}{\rule{\linewidth}{0.4pt}}\par\vspace{5pt}}

\newcommand{\findbadge}[1]{%
  \tcbox[on line, nobeforeafter, box align=base,
    boxsep=0pt, left=4pt, right=4pt, top=1pt, bottom=1pt,
    arc=2pt, boxrule=0pt,
    colback=conftriage-blue, colframe=conftriage-blue]{%
    \color{white}\bfseries\footnotesize #1}}

\newcommand{\finding}[3]{%
  \par\noindent\findbadge{#1}\enspace\textbf{#2}\par
  \vspace{2pt}%
  \begingroup\setlength{\leftskip}{1.5em}\noindent #3\par\endgroup}

\newcommand{\findsep}{\par\vspace{5pt}%
  \textcolor{conftriage-blue!35}{\rule{\linewidth}{0.4pt}}\par\vspace{5pt}}

\begin{document}

\title{ConfTriage: A Calibration-Aware LLM Triage Framework for Pulmonary Nodule Malignancy with Selective Specialist Deferral}

\author{
\IEEEauthorblockN{Md Rabiul Islam$^{*}$, Samir Abdaljalil$^{*}$, Erchin Serpedin,~\IEEEmembership{Fellow,~IEEE}, and Hasan Kurban}
\thanks{$^{*}$M.R. Islam and S. Abdaljalil contributed equally to this work.}
\thanks{M.R. Islam, S. Abdaljalil, and E. Serpedin are with the Department of Electrical and Computer Engineering, Texas A\&M University, College Station, TX, USA. E\textendash mail: \{rabiul\_islam, sabdaljalil, eserpedin\}@tamu.edu.}
\thanks{H. Kurban is with the College of Science and Engineering, Hamad Bin Khalifa University, Doha, Qatar. E\textendash mail: hkurban@hbku.edu.qa.}
}

\maketitle

\begin{abstract}
Pulmonary nodule malignancy prediction typically depends on image-trained specialist deep learning (DL) models that require substantial annotated imaging data and task-specific training. We investigate whether a generalist large language model (LLM), reading only a faithful natural\textendash language rendering of standard nodule attributes, can serve as a calibrated triage layer. To address this question, we proceed in three steps. \textbf{(1) Framework:} We propose \method{}, a confidence\textendash calibrated method built on three pillars: language as the modality (the LLM consumes structured\textendash to\textendash language input rather than raw pixels), calibration as the safety mechanism (verbalized confidences are transformed in logit space and Platt\textendash scaled), and a selective specialist DL backstop for low\textendash confidence cases. \textbf{(2) Theory:} We prove two guarantees: a finite\textendash sample combined\textendash error bound that yields an explicit per\textendash threshold operational certificate (Theorem~\ref{thm:triage}); an oracle inequality showing that excess risk over the Bayes\textendash optimal deferral classifier is controlled by the $L^1$ calibration error of the LLM probability, formally tying calibration to triage optimality (Theorem~\ref{thm:oracle}). \textbf{(3) Evidence:} A controlled seven\textendash way input ablation across five frontier LLMs (Gemini\textendash 3.1\textendash Flash\textendash Lite, Mistral\textendash Large, Qwen\textendash 2.5\textendash 72B, Claude\textendash 3.5\textendash Haiku, GPT\textendash 4o\textendash mini) on LIDC\textendash IDRI shows that natural\textendash language descriptions dominate the diagnostic signal, that low\textendash level image statistics are essentially diagnostically vacuous. The proposed \method{} achieved an F1 score of 88.22\% and an AUC of 0.92, while resolving 76.5\% of cases using zero-shot LLM inference alone and referring only the remaining uncertain cases to the specialist DL backstop. These results demonstrate that clinically meaningful diagnostic information can be effectively captured through structured radiological descriptions and leveraged by calibrated LLMs for selective referral. More broadly, the proposed framework suggests a practical pathway for combining generalist LLM prediction with specialist artificial intelligence (AI) models in medical decision-support systems. The source code of ConfTriage is publicly available at
\underline{\href{https://github.com/rabiul-ai/ConfTriage}{github.com/rabiul-ai/ConfTriage}}.

\end{abstract}


\section{Introduction}
\label{sec:intro}

Low\textendash dose computed tomography (CT) screening for lung cancer detects vast numbers of pulmonary nodules, the majority of which are benign \cite{armato2011lung}. Triaging these nodules into those that require specialist follow\textendash up and those that do not is a high\textendash volume, time\textendash sensitive task that creates substantial workload for thoracic radiologists. Specialized DL models trained end\textendash to\textendash end on nodule volumes have reached accuracies of 91\% and above on the standard LIDC\textendash IDRI benchmark \cite{xie2018knowledge, shi2021semi, wu2023self}, but these systems require image\textendash level training, are opaque to clinicians, and ship with no native interface for explanation in clinical language. Generalist LLMs offer a complementary capability: they consume structured and free\textendash text descriptions of a finding and return both a diagnosis and a rationale in language. A recent work \cite{mao2025assessments} has shown that GPT\textendash 4o can reach an average accuracy of 0.88 on longitudinal lung\textendash nodule malignancy prediction, hinting that the linguistic affordances of LLMs may be useful for triage. Whether such systems are reliable enough to deploy, what input modality actually drives their predictions, and whether their confidences can be calibrated for safe operational use remain open questions.

\textbf{The gap:} For pulmonary nodule malignancy classification, specialized supervised convolutional neural network (CNN) \cite{xie2018knowledge, shi2021semi, wu2023self} and vision transformer (ViT) \cite{veasey2025low} pipelines have achieved high accuracy but require task-specific training and offer limited linguistic interpretability. Vision\textendash language pipelines built on contrastive language--image pretraining (CLIP) and segment anything model (SAM) variants \cite{shaukat2024lung, zhuang2025vision} reach an area under the receiver operating characteristic curve (AUC) of approximately 0.90 on the National Lung Screening Trial cohort but degrade under distribution shift to LUNGx and provide no controlled disentanglement of input\textendash modality contributions. Generalist LLM evaluations in radiology \cite{mao2025assessments, brin2025assessing, sonoda2024diagnostic, hu2026benchmarking} report headline accuracy figures from single models on heterogeneous tasks, without isolating which information channel carries the signal, without producing a calibrated triage policy, and without comparing across vendors on a single benchmark with rigorous statistical machinery. Practitioners deciding whether to use an LLM as a triage layer therefore have neither (i) a controlled measurement of how much of the signal is linguistic versus visual, nor (ii) a deployable method that turns LLM confidence into a defensible deferral policy with formal guarantees.

\textbf{The insight:} The proposed \method{} is a methodological framework rather than a single model. Two empirical observations and one theoretical observation, taken together, motivate a structured framework for how clinical AI for triage might be organized in the LLM era. \emph{First (empirical),} standard radiology workflow already produces a small set of structured nodule descriptors (subtlety, internal structure, calcification, sphericity, margin, lobulation, spiculation, texture). If those descriptors carry most of the diagnostic signal, then a generalist LLM that operates on a faithful language rendering of them does not need to consume the CT volume directly. \emph{Second (empirical),} this property is not specific to a single vendor: if it holds for one frontier LLM, it appears to hold across multiple frontier LLMs from independent providers, which is operationally important because deployments cannot lock to a single closed model. \emph{Third (theoretical),} a calibrated LLM probability is not just a presentation aid; it is the formal object that makes a deferral policy near\textendash optimal in a precise sense (Theorem~\ref{thm:oracle}). Together, these observations support a framework with three pillars: (P1) language as the modality (the LLM consumes a structured\textendash to\textendash language rendering rather than raw pixels); (P2) calibration as the safety mechanism (verbalized confidences are transformed in logit space and Platt\textendash scaled); and (P3) a selective specialist backstop as the formal certificate (low\textendash confidence cases route to a specialist DL classifier, with provable bounds on combined error). All three pillars are empirically falsifiable. In this paper, we test P1 directly through the controlled input ablation. We develop the theoretical scaffolding for P2 and P3 (Theorems~\ref{thm:triage}, \ref{thm:oracle}). The resulting combined operating curve, obtained through Platt scaling on $\Dcalib$ and a deferral\textendash threshold sweep, extends the uncertain\textendash case rejection framework of Certain-Net\cite{islam2026certainnet} into a calibrated selective\textendash prediction setting.

\textbf{Approach:} To support the framework empirically and theoretically, we (i)~construct a deterministic input--construction function $\phimap$ that produces seven input regimes spanning `radiology attributes only,' `language description only,' `image statistics only,' and four combinations, and evaluate five frontier LLMs from independent vendors on the LIDC--IDRI benchmark under each regime; (ii)~introduce \method{}, a confidence--calibrated triage framework whose two theorems supply, respectively, an operational error budget and a calibration--to--optimality bridge; and (iii)~perform attribute--redaction and synthetic--corruption controls together with an evaluation against leave--one--out (LOO) consensus labels derived from the available LIDC--IDRI reader annotations, thereby quantifying model performance relative to expert-reader variability (Section~\ref{sec:reader_variability}). All analyses use only publicly available de--identified data; statistical testing uses paired bootstrap with Benjamini--Hochberg false--discovery--rate (FDR) control across all reported comparisons.

Overall, the major contributions of our research are as follows.
\begin{enumerate}[leftmargin=*, itemsep=2pt, topsep=2pt]
\item Method: We articulate a three\textendash pillar methodological framework for clinical triage with generalist LLMs (language as modality, calibration as safety mechanism, selective backstop as certificate) and instantiate it as \method{}, a confidence\textendash calibrated method using logit\textendash space of verbalized confidences, Platt scaling on a held\textendash out fold, and threshold\textendash based deferral to a specialist DL backstop (Section~\ref{sec:method}, Algorithm~\ref{alg:conftriage}).

\item Formal guarantees: We prove two theoretical results supporting the proposed framework. Theorem~\ref{thm:triage} provides a finite\textendash sample combined\textendash error bound that yields a per\textendash threshold operational certificate. Theorem~\ref{thm:oracle} establishes an oracle inequality showing that the excess risk of \method{} relative to the Bayes\textendash optimal deferral classifier is controlled by the $L^1$ calibration error of the LLM probability, formally connecting probability calibration to triage optimality.

\item Empirical findings: A controlled seven\textendash way input ablation across five frontier LLMs on LIDC\textendash IDRI demonstrates that natural\textendash language descriptions dominate the LLM diagnostic signal, that image\textendash statistics inputs are essentially diagnostically vacuous, and that a small subset of standard LIDC\textendash IDRI attributes localizes most of the gap. Significance is assessed using paired bootstrap with Benjamini\textendash Hochberg FDR correction (Section~\ref{sec:results}).

\item Reliability and robustness evaluation: Comprehensive evaluations are performed to assess the reliability of \method{} through LOO reader-consensus evaluation and calibration analysis, while robustness of the language-based diagnostic signal is examined through synthetic corruption controls, description-source ablation, and prompt-temperature sensitivity studies.
\end{enumerate}

\textbf{Key Distinctions:} The closest prior work \cite{mao2025assessments} evaluates GPT\textendash 4o on longitudinal CT for nodule assessment and reports 0.88 average accuracy. Our contribution differs along three specific axes: (i)~we propose a named, deployable triage method with calibration, deferral, and finite\textendash sample guarantees, rather than reporting raw model accuracy; (ii)~we perform a controlled seven\textendash way input ablation that disentangles which input channel drives the LLM signal, which Mao et~al.~did not attempt; and (iii)~we evaluate five vendors rather than one, which is operationally relevant because deployments cannot lock to a single closed model.

\begin{tcolorbox}[
  enhanced, breakable,
  colback=conftriage-lightblue!30,
  colframe=conftriage-blue,
  boxrule=0.6pt, arc=2.5pt,
  left=8pt, right=8pt, top=5pt, bottom=7pt,
  before skip=6pt, after skip=6pt,
  title={Key Findings at a Glance},
  coltitle=white, colbacktitle=conftriage-blue,
  fonttitle=\bfseries\small, toptitle=2pt, bottomtitle=2pt
]

\finding{1}{Linguistic sufficiency.}{%
Language-based descriptions provide sufficient diagnostic information for LLM-based pulmonary nodule classification and consistently outperform image-statistics inputs across all evaluated frontier LLMs (see Table~\ref{tab:main_results}).}

\findsep

\finding{2}{Image statistics are diagnostically vacuous.}{%
Across all five LLMs, image\textendash statistics input collapses to AUC in $[0.443, 0.576]$, with several models defaulting to a single class (F1 $= 0$). The mean text\textendash to\textendash image AUC gap is $\Delta = 0.358$.}

\findsep

\finding{3}{Formal guarantees turn confidence into a contract.}{%
\method{} packages the LLM into a calibrated triage layer with a specialist DL backstop. Theorem~\ref{thm:triage} supplies a finite\textendash sample combined\textendash error certificate; Theorem~\ref{thm:oracle} an oracle inequality tying $L^{1}$ calibration error to excess risk over the Bayes\textendash optimal deferral classifier.}

\end{tcolorbox}

\section{Related Work}
\label{sec:related}

\subsection{Specialized DL for Nodule Classification}
MV\textendash KBC \cite{xie2018knowledge} decomposes a 3D nodule into nine fixed views and integrates appearance, voxel, and shape submodels through ResNet50 backbones; on LIDC\textendash IDRI it reaches 91.60\% accuracy and AUC of 0.957. SDTL \cite{shi2021semi} couples nodule\textendash vs\textendash tissue pretraining with iterative feature\textendash matching pseudo\textendash labeling. Visual\textendash attention\textendash driven self\textendash supervised transfer learning \cite{wu2023self} and LoRA adaptation of pretrained large vision models \cite{veasey2025low} continue this line and report further accuracy gains. More recently, Gunawan \emph{et al.}~\cite{gunawan2024combining} proposed a multistage preprocessing pipeline combining 3D filtering with a modified segmentation network, reporting high classification accuracy through extensive image preprocessing and supervised CNN training. These methods are accurate but each requires nontrivial image\textendash level training and does not produce native linguistic interfaces for clinician auditing. Our approach is operationally distinct: we test how far a generalist LLM operating on structured\textendash to\textendash language input can go and use specialized DL only as a backstop for low\textendash confidence cases.

\subsection{Vision\textendash Language Pipelines and LLM Evaluation} Lung\textendash CADex \cite{shaukat2024lung} couples MedSAM zero\textendash shot segmentation with CLIP prefix tuning and reports detection sensitivity of 0.86 on LIDC\textendash IDRI compared with 0.76 for fully supervised baselines, with external evaluation on LUNGx. The semantic\textendash guided vision\textendash language model (VLM) of \cite{zhuang2025vision} aligns CT image features with sentence\textendash level radiologist semantic features generated via Gemini and reaches AUC 0.901 on the NLST test split. Both methods integrate vision and language but do not isolate the contribution of language alone, do not emit calibrated confidences, and do not implement a deferral policy.

Few studies \cite{ghorbian2025comprehensive, wu2024collaborative} evaluate the LLMs in Radiology. For instances, 
Brin et~al.~\cite{brin2025assessing} evaluate GPT\textendash 4V on 230 emergency\textendash room images and report pathology\textendash recognition accuracy of 35.2\%. Sonoda et~al.~\cite{sonoda2024diagnostic} compare GPT\textendash 4o, Claude 3 Opus, and Gemini 1.5 Pro on 324 `Diagnosis Please' cases. Hu et~al.~\cite{hu2026benchmarking} benchmark GPT\textendash 5 against GPT\textendash 4o on VQA\textendash RAD, SLAKE, and a curated medical\textendash physics multiple\textendash choice set. Mao et~al.~\cite{mao2025assessments} provide the most directly relevant prior result: GPT\textendash 4o reaches 0.88 average accuracy on longitudinal nodule malignancy assessment over 647 patients. None of these papers performs a controlled multi-modality ablation, packages the results as a calibrated triage method with formal guarantees, or provides multi-vendor comparisons under FDR-corrected statistical testing on a single nodule benchmark.

\subsection{Selective Classification and Learning to Defer} \method{} is, mathematically, a selective classifier with an external backstop. The selective classification framework of Geifman and El\textendash Yaniv \cite{geifman2017selective} formalizes coverage\textendash risk trade\textendash offs for classifiers that abstain; the learning\textendash to\textendash defer framework of Madras et~al.~\cite{madras2018predict} extends this to systems that route to a downstream expert. We are not aware of a prior instantiation of either framework for an LLM\textendash plus\textendash specialist hybrid in clinical triage. Our contribution within this line is the combination of a calibrated probability, formal guarantees for selective referral, and a public release suited for prospective external validation.


\section{Method: \method{}}
\label{sec:method}

\subsection{Problem Formulation}
\label{sec:setup}

Let a single lung nodule case be the tuple $\xfeat = (\ctv, \attr, \desc) \in \Xcal$, where $\ctv \in \Real^{H \times W \times D}$ is the cropped 3D CT volume centered at the nodule, $\attr \in \Real^{k}$ is a vector of $k$ structured radiology attributes, and $\desc$ is a deterministic natural\textendash language description of the nodule. We use the eight standard LIDC\textendash IDRI nodule\textendash characteristic attributes ($k = 8$): subtlety, internal structure, calcification, sphericity, margin, lobulation, spiculation, and texture. In addition, nodule diameter ($d$) and estimated anatomical position derived from the nodule centroid are included. Let $\labelvar \in \{0, 1\}$ denote the binary benign\textendash versus\textendash malignant target, and let $\Pdist$ denote the joint distribution over $(\xfeat, \labelvar)$ induced by the source population. We seek a function that, given $\xfeat$, outputs either a calibrated probability estimate $\probhat \in [0,1]$ of the Bayes posterior $\eta(\xfeat) = \Pr_{\Pdist}[\labelvar = 1 \mid \xfeat]$ or a deferral signal handing the case to a specialist DL classifier $\specdl$.

\subsection{Preprocessing}
\label{sec:preprocessing}

From each lung nodule, we extracted nine 2D views. We also generated a structured textual description from the radiological attributes, nodule diameter, and anatomical location estimated from the nodule centroid ($x, y, z$), as shown in Fig.~\ref{fig:preprocessing}.

\subsubsection{CT}
CT volumes undergo standard preprocessing (HU correction, $1{\times}1{\times}1$\,mm$^3$ resampling, $32{\times}32{\times}32$\,mm$^3$ centered crop, intensity windowing at level $-200$ and width $1200$, normalization to $[0,1]$); from each crop, three orthogonal anatomical and six diagonal slices are extracted.

\begin{figure}[tb]
\centering
\includegraphics[width=\columnwidth]{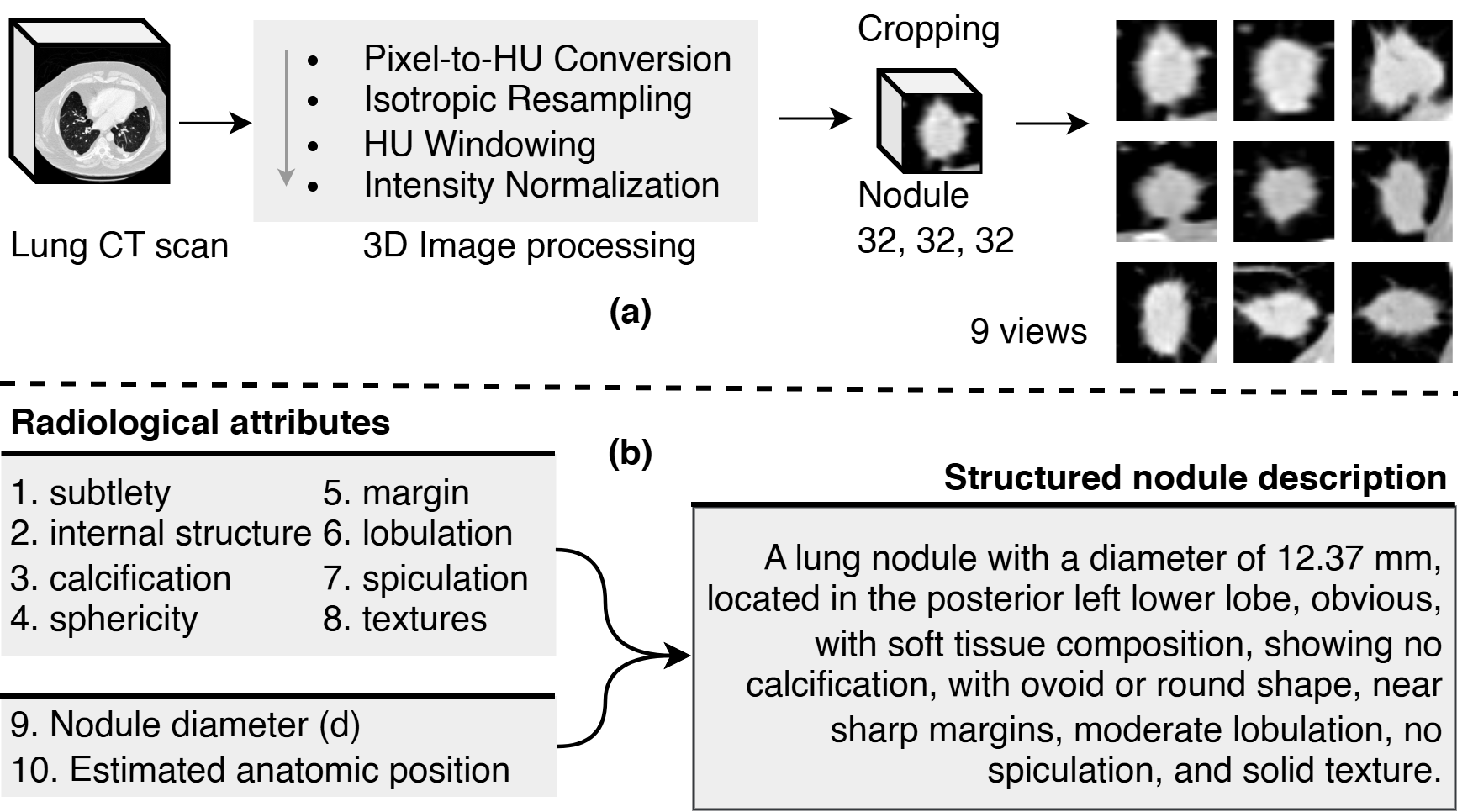}
\caption{(a) CT preprocessing for generating nine 2D views from each 3D nodule. (b) Structured nodule description generation from radiological attributes, nodule diameter, and anatomical location.}
\label{fig:preprocessing}
\end{figure}

\subsubsection{Text}
The natural\textendash language description $\desc = T(\attr)$ is generated deterministically from $\attr$ using a fixed template $T$ defined in Algorithm~\ref{alg:descgen}. The generator is reproducible, blind to the label, and released under the public GitHub repository. 

\begin{algorithm}[tb]
\caption{Deterministic description generator $T(\attr)$}
\label{alg:descgen}
\begin{algorithmic}[1]
\REQUIRE Structured nodule descriptors $\attr = (d,p,a_1,\ldots,a_8)$, where $d$ denotes nodule diameter (mm), $p$ denotes the anatomical position estimated from the nodule centroid, and $(a_1,\ldots,a_8)$ are the standard LIDC\textendash IDRI radiological attributes.
\STATE Map the anatomical position code $p \in \{1,\ldots,10\}$ to its canonical anatomical descriptor $\lambda_p$ using the predefined location dictionary (e.g., $p=1 \rightarrow$ ``posterior left lower lobe'').
\STATE Map each radiological attribute $a_j$ to its canonical lexical token $\lambda_j$ using the LIDC\textendash IDRI dictionary; for example, $a_5 \in {1,\ldots,5}$ for margin maps to a five\textendash point scale from `poorly defined'' (1) to `sharp'' (5).
\STATE Concatenate into the fixed sentence: ``A lung nodule with a diameter of $d$ mm, located in the $\lambda_p$, $\lambda_1$ (subtlety), with $\lambda_2$ (internal structure), showing $\lambda_3$ (calcification), $\lambda_4$ (sphericity) shape, $\lambda_5$ margins, $\lambda_6$ (lobulation), $\lambda_7$ (spiculation), and  $\lambda_8$ (texture).''
\STATE \textbf{return} resulting string $\desc$.
\end{algorithmic}
\end{algorithm}

\subsection{Seven Input Regimes}
\label{sec:seven_regimes}

We define seven deterministic input\textendash construction functions $\phimap$, each producing a different prompt to the LLM:
\begin{enumerate}[leftmargin=*, itemsep=1pt, topsep=2pt]
\item $\phimap_{\mathrm{rad}}(\xfeat) = \attr$ (numeric attributes only, JSON\textendash serialized).
\item $\phimap_{\mathrm{text}}(\xfeat) = \desc$ (natural\textendash language description only).
\item $\phimap_{\mathrm{img}}(\xfeat) = \imgstats(\ctv)$, where $\imgstats$ returns intensity histogram summaries (mean, standard deviation, skew, kurtosis, percentiles) and gray\textendash level co\textendash occurrence matrix features (contrast, homogeneity, energy, correlation).
\item $\phimap_{\mathrm{rad+text}}(\xfeat) = (\attr, \desc)$.
\item $\phimap_{\mathrm{rad+img}}(\xfeat) = (\attr, \imgstats(\ctv))$.
\item $\phimap_{\mathrm{text+img}}(\xfeat) = (\desc, \imgstats(\ctv))$.
\item $\phimap_{\mathrm{all}}(\xfeat) = (\attr, \desc, \imgstats(\ctv))$.
\end{enumerate}


\begin{figure*}[!t]
  \centering
\includegraphics[width=\textwidth]{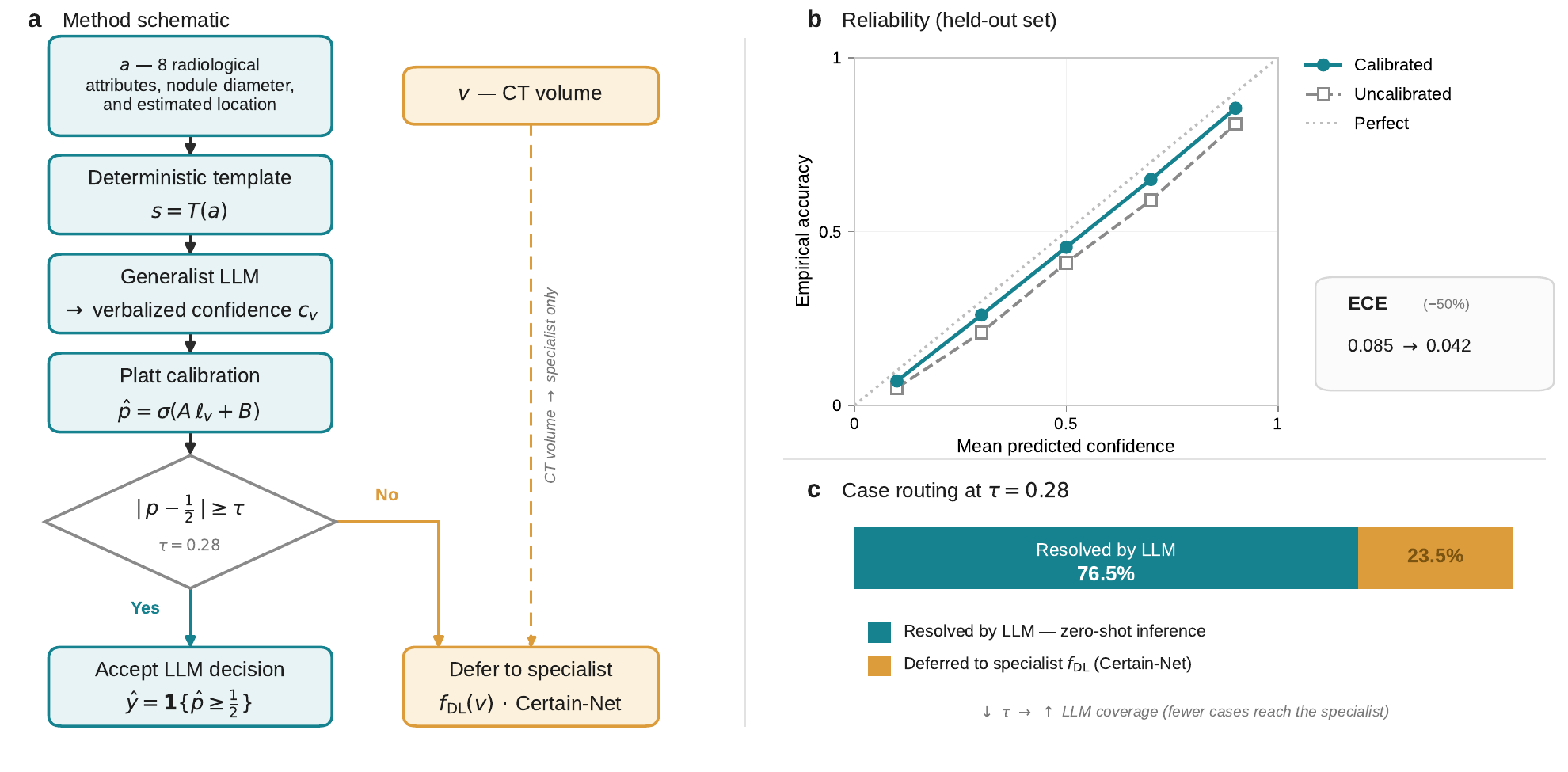}

  \caption{The \method{} triage pipeline and its operating behavior. (a) Pipeline of the confidence-aware triage system integrating a generalist LLM with a specialist DL classifier. (b) Reliability diagram showing that uncalibrated confidences are overconfident, whereas Platt scaling reduces the expected calibration error (ECE) from $0.085$ to $0.042$. (c) At the operating threshold $\tau = 0.28$, the LLM resolves $76.5\%$ of cases, while $23.5\%$ are deferred to the specialist. Lowering $\tau$ increases LLM coverage.}
  \label{fig:conftriage}
\end{figure*}

\subsection{Overview: \method{}}
\label{sec:method_overview}

The \method{} framework takes a test case $\xfeat$ as input and produces either a calibrated LLM probability $\probhat$ or a deferral to the specialist DL classifier $\specdl(\ctv)$, as outlined in Algorithm~\ref{alg:conftriage}. Fig. ~\ref{fig:conftriage}(a) summarizes the overall pipeline of \method{}. The framework consists of three components.

(1) Input construction: Build $\phimap_{\mathrm{text}}(\xfeat) = T(\attr)$ using the deterministic template (Algorithm~\ref{alg:descgen}). The empirical results in Section~\ref{sec:results} show that $\phimap_{\mathrm{text}}$ is the highest\textendash performing input regime; \method{} therefore uses it as the LLM input. (The choice is principled: the comparison is reproducible and the input is small, so the LLM's API cost stays modest.)

(2) Confidence transformation in logit space: Query the LLM under a locked prompt template and collect the verbalized confidence $\verbconf \in (0,1)$. The verbalized confidence is converted into logit space as
\begin{equation}
\ell_{v} = \mathrm{logit}(\verbconf)  , 
\label{eq:logitforms}
\end{equation}
where $\mathrm{logit}(z) = \log(z / (1 - z))$ for $z \in (0, 1)$. 
Transforming confidence into logit space yields an unbounded score representation that is well suited for subsequent probabilistic calibration.

(3) Calibration and routing:
Platt scaling \cite{platt1999probabilistic} and temperature scaling, the single\textendash parameter variant analyzed in \cite{guo2017calibration}, are standard tools for converting raw model scores into well\textendash calibrated probabilities. We follow this line and adapt it to the situation of generalist LLMs that consider the verbalized confidence generated by LLMs.

A Platt\textendash scaling map $\plattmap : \Real \to [0, 1]$ is fitted on $\Dcalib$ to convert $\ell_{\mathrm{v}}$ into a calibrated probability $\probhat$:
\begin{equation}
\probhat = \plattmap(\ell_{\mathrm{v}}) = \sigmoidf\bigl(A\, \ell_{\mathrm{v}} + B\bigr),
\label{eq:platt}
\end{equation}
with $(A, B)$ optimized using the Limited-memory Broyden--Fletcher--Goldfarb--Shanno (L\textendash BFGS) algorithm on the calibration\textendash fold cross\textendash entropy. A threshold $\thresh \in [0,1/2]$ is selected on $\Dcalib$ by sweeping candidate threshold values and evaluating the resulting trade-off between coverage and predictive performance. Cases with $|\probhat - 1/2| < \thresh$ are deferred to $\specdl$. The deferral indicator is $\defind = \indicator\{|\probhat - 1/2| < \thresh\}$.

\begin{algorithm}[t]
\caption{\method{}: confidence\textendash calibrated LLM triage with specialist backstop}
\label{alg:conftriage}
\begin{algorithmic}[1]
\REQUIRE Test case $\xfeat$; LLM $\LLM$; specialist DL classifier $\specdl$; and Platt parameters $(A, B)$ fitted on $\Dcalib$; threshold $\thresh$.
\STATE Build the language input $\phimap_{\mathrm{text}}(\xfeat) = T(\attr)$ via Algorithm~\ref{alg:descgen}.
\STATE Query $\LLM$ with the locked prompt; collect $\verbconf$ (verbalized confidence in $(0, 1)$).
\STATE Compute logits $\ell_{v} = \mathrm{logit}(\verbconf)$.
\STATE Calibrate: $\probhat \gets \sigmoidf(A \ell_{\mathrm{v}} + B)$.
\IF{$|\probhat - 1/2| \geq \thresh$}
\STATE \textbf{return} LLM prediction $\indicator\{\probhat \geq 1/2\}$ with calibrated probability $\probhat$ and $\defind \gets 0$.
\ELSE
\STATE \textbf{return} $\specdl(\ctv)$ (defer to specialist) and $\defind \gets 1$.
\ENDIF
\end{algorithmic}
\end{algorithm}

\subsection{LLMs and Spealist Backstop}
\subsubsection{Multi\textendash vendor LLM Pool}
\label{sec:vendors}

We evaluate five frontier LLMs from distinct providers, balanced between closed and open\textendash source families: OpenAI GPT\textendash 4o\textendash mini, Anthropic Claude\textendash 3.5\textendash Haiku, and Google Gemini\textendash 3.1\textendash Flash\textendash Lite (closed, accessed via vendor APIs); Mistral\textendash Large and Alibaba Qwen\textendash 2.5\textendash 72B (open\textendash weights, served through hosted endpoints). The intent of this pool is twofold: to test whether the structured textual input based performance is vendor\textendash specific or general, and to enable a closed\textendash versus\textendash open\textendash source comparison along the lines reported in Section~\ref{sec:res_main}. Each model is queried with temperature 0 (deterministic) under a locked prompt template; verbalized confidence is captured for the calibration analysis.

\subsubsection{Specialist Backstop: Certain-Net}
When the LLM-based prediction is not sufficiently confident (i.e., $|\hat{p} - \tfrac{1}{2}| < \tau$), we defer the case to the specialist DL classifier $f_{\mathrm{DL}}(v)$. The specialist backstop $\specdl$ is our recently proposed Certain-Net \cite{islam2026certainnet} model, an uncertainty-aware ensemble of Monte Carlo (MC) dropout convolutional neural networks built upon the EfficientNetV2B0 backbone and trained using a nine-view 2D representation of each pulmonary nodule.

Certain-Net achieved an accuracy of 95.63\% when making predictions on only the most certain 55.08\% of cases. However, achieving this performance requires task-specific image-based model training. Motivated by the desire to reduce this training burden, we investigate whether an LLM can accurately resolve a substantial fraction of cases through zero-shot prediction, reserving only uncertain cases for a specialist backstop. This motivation led to the development of the proposed ConfTriage framework.

\subsection{Leave--One--Out Consensus Protocol}
\label{sec:LOO_protocol}
Consider a nodule annotated by readers $\{r_1,r_2,\ldots,r_m\}$ with malignancy ratings $\{y_{i1},y_{i2},\ldots,y_{im}\}$. When reader $r_j$ is selected as the test reader, the remaining readers form the consensus:

\begin{equation}
C_{ij}
=
\operatorname{Median}
\left(
\{y_{ik}: k\neq j\}
\right).
\end{equation}

Following standard LIDC\textendash IDRI practice, ratings below 3 were mapped to the benign class, ratings above 3 were mapped to the malignant class, and ambiguous ratings equal to 3 were excluded. The binary LOO consensus label is therefore

\begin{equation}
\tilde{y}_{ij}
=
\begin{cases}
1, & C_{ij}>3,\\
0, & C_{ij}<3.
\end{cases}
\end{equation}
Thus, each reader is evaluated against the collective opinion of the remaining readers. For example,
\begin{equation}
R_1 \leftrightarrow \operatorname{Consensus}(R_2,R_3,R_4),
\end{equation}
with analogous comparisons for all available readers. This protocol provides a principled estimate of individual-reader agreement with multi-reader consensus while avoiding assumptions regarding globally consistent reader identities.

\section{Theoretical Analysis}
\label{sec:theory}

The theoretical analysis makes two claims, each supporting an operational property of the proposed paradigm. Theorem~\ref{thm:triage} bounds the combined error of \method{} (Algorithm~\ref{alg:conftriage}) and is the formal object underlying the triage operating curve. Theorem~\ref{thm:oracle} is an oracle inequality that bounds \method{}'s excess risk over the Bayes\textendash optimal deferral classifier in its class by the $L^1$ calibration error of the LLM probability, formally connecting Pillar~P2 (calibration as the safety mechanism) to Pillar~P3 (selective backstop as the certificate).

\subsection{Coverage\textendash Risk Decomposition for \method{}}
\label{sec:coverage_risk}

\begin{assumption}[i.i.d.\ held-out test fold]
\label{ass:iid}
In each cross-validation iteration, the held-out test fold
$\Dtest=\{(\xfeat_i,\labelvar_i)\}_{i=1}^n$
consists of independent samples drawn from the underlying data distribution $\Pdist$.
\end{assumption}

Under Assumption~\ref{ass:iid}, define the empirical retained\textendash case error $\hat{\Risk}_{\LLM}(\thresh)$ and the empirical deferred\textendash case error $\hat{\Risk}_{\specdl}(\thresh)$:
\begin{align}
\hat{\Risk}_{\LLM}(\thresh) &= \frac{1}{n}\sum_{i=1}^n (1 - \defind_i)\, \indicator\{\predY_i^{\LLM} \neq \labelvar_i\}, \label{eq:Rllm}\\
\hat{\Risk}_{\specdl}(\thresh) &= \frac{1}{n}\sum_{i=1}^n \defind_i\, \indicator\{\predY_i^{\specdl} \neq \labelvar_i\}, \label{eq:Rdl}
\end{align}
where $\predY_i^{\LLM} = \indicator\{\probhat_i \geq 1/2\}$ and $\predY_i^{\specdl} = \specdl(\ctv_i)$. The combined ConfTriage prediction is $\predY_i = (1 - \defind_i)\predY_i^{\LLM} + \defind_i\, \predY_i^{\specdl}$, and the combined population risk is $\Risk(\predY) = \Pr_{\Pdist}[\predY \neq \labelvar]$.

Intuitively, the next theorem says that the combined error of \method{} on the population is, with high probability, bounded by the sum of its empirical retained and deferred errors plus a Hoeffding correction. This is the formal object on which the operating\textendash curve analysis in Section~\ref{res:operating_threshold} rests.

\begin{theorem}[Finite\textendash sample combined risk bound]
\label{thm:triage}
Under Assumption~\ref{ass:iid}, for any threshold $\thresh \in [0, 1/2]$ and any $\delta \in (0, 1)$, with probability at least
$1-\delta$
over the draw of the held-out test fold $\Dtest$,
\begin{equation}
\Risk(\predY) \leq \hat{\Risk}_{\LLM}(\thresh) + \hat{\Risk}_{\specdl}(\thresh) + \sqrt{\frac{\log(1/\delta)}{2n}}.
\label{eq:triage_bound}
\end{equation}
\end{theorem}

\begin{remark}
The bound is operational: for any fixed threshold $\thresh$, the right--hand side is computable from the held--out test fold and therefore provides a finite--sample certificate of the combined error. Because ConfTriage is evaluated using five--fold cross--validation, the certificate is computed independently on each held--out fold using the common operating threshold selected from the pooled out--of--fold predictions. This preserves the assumptions of Theorem~\ref{thm:triage} while providing a deployment--relevant certificate for each cross--validation split.
\end{remark}

\begin{proof}[Proof sketch of Theorem~\ref{thm:triage}]

Set $Z_i = (1 - \defind_i)\, \indicator\{\predY_i^{\LLM} \neq \labelvar_i\} + \defind_i\, \indicator\{\predY_i^{\specdl} \neq \labelvar_i\}$. Each $Z_i\in\{0,1\}$ and, for the current held-out cross-validation fold, the $Z_i$ are i.i.d. under Assumption~\ref{ass:iid}. By linearity of expectation and the disjoint decomposition $1 = (1 - \defind_i) + \defind_i$,
\begin{equation*}
\mathbb{E}_{\Pdist}[Z_1] = \Pr_{\Pdist}[\predY \neq \labelvar] = \Risk(\predY).
\end{equation*}
Hoeffding's inequality applied to bounded i.i.d.\ variables yields, for any $t > 0$,
$\Pr\bigl(\mathbb{E}[Z_1] - \bar{Z} > t\bigr) \leq \exp(-2 n t^2)$,
where $\bar{Z} = n^{-1}\sum_i Z_i = \hat{\Risk}_{\LLM}(\thresh) + \hat{\Risk}_{\specdl}(\thresh)$. Setting the right\textendash hand side equal to $\delta$ and solving gives $t = \sqrt{\log(1/\delta)/(2n)}$, which yields the one\textendash sided upper bound on $\Risk(\predY)$ in~\eqref{eq:triage_bound}. (A symmetric two\textendash sided bound on $|\mathbb{E}[Z_1] - \bar{Z}|$ would replace $\log(1/\delta)$ by $\log(2/\delta)$ via union bound; we use the one\textendash sided form because we only need an upper bound on the risk.)
\end{proof}

\begin{remark}[Tightness]
Theorem~\ref{thm:triage} uses no structural information about the joint distribution of $(\probhat, \labelvar, \specdl(\ctv))$ and is therefore the tightest \emph{distribution\textendash free} bound that can be obtained from the empirical risk. Tighter bounds are available under additional assumptions: under bounded variance of $Z_i$, Bernstein's inequality replaces the slack with $\mathcal{O}(\sqrt{V_n \log(1/\delta)/n} + \log(1/\delta)/n)$ where $V_n$ is the empirical variance; under realizability or bounded VC dimension on the deferral rule, uniform convergence over $\thresh$ replaces the slack with $\mathcal{O}(\sqrt{(\log K + \log(1/\delta))/n})$ for $K$ candidate thresholds. We report the simpler bound here for transparency.
\end{remark}





\subsection{Oracle Inequality: Calibration Controls Excess Risk}
\label{sec:oracle}

Theorem~\ref{thm:triage} certifies the combined error rate but says nothing about whether \method{} is using the available signal \emph{well} relative to a Bayes\textendash optimal benchmark. We now show that the gap between \method{}'s combined risk and the risk of the Bayes\textendash optimal deferral classifier in its class is controlled by the $L^1$ calibration error of the LLM probability. This is the formal bridge between calibration (a measurable diagnostic) and triage optimality (a target property).

\textbf{Bayes baseline.} Let $\eta(\xfeat) = \Pr_{\Pdist}[\labelvar = 1 \mid \xfeat]$ be the Bayes posterior. For a fixed deferral threshold $\thresh$ and a fixed specialist backstop $\specdl$, define the \emph{Bayes\textendash optimal deferral classifier with backstop $\specdl$} by
\begin{equation}
\predY^*_{\thresh}(\xfeat) =
\begin{cases}
\indicator\{\eta(\xfeat) \geq 1/2\} & \text{if } |\eta(\xfeat) - 1/2| \geq \thresh,\\
\specdl(\ctv) & \text{otherwise},
\end{cases}
\label{eq:bayes_oracle}
\end{equation}
with risk $\Risk(\predY^*_{\thresh}) = \Pr_{\Pdist}[\predY^*_{\thresh} \neq \labelvar]$. The classifier $\predY^*_{\thresh}$ has perfect knowledge of the posterior on retained cases and uses the same backstop on deferred cases as \method{}; it is the strongest baseline that respects the same operating constraint. Define the $L^1$ calibration error
\begin{equation}
\Delta(\probhat) = \mathbb{E}_{\Pdist}\bigl|\probhat(\xfeat) - \eta(\xfeat)\bigr|,
\label{eq:l1cal}
\end{equation}
the expected pointwise distance from the Bayes posterior, and the deferral discrepancy probability $\pi_{\Delta}(\thresh) = \Pr_{\Pdist}[\hat{\defind}(\xfeat) \neq \defind^*(\xfeat)]$ where $\hat{\defind}(\xfeat) = \indicator\{|\probhat - 1/2| < \thresh\}$ and $\defind^*(\xfeat) = \indicator\{|\eta - 1/2| < \thresh\}$.

Intuitively, the next theorem says that any \method{}\textendash style policy whose probability is close to the Bayes posterior in $L^1$ also has near\textendash Bayes\textendash optimal risk, with two interpretable sources of slack: $L^1$ calibration error and deferral\textendash boundary disagreement.

\begin{theorem}[Excess risk of \method{} versus Bayes\textendash optimal deferral]
\label{thm:oracle}
Let $\predY_{\thresh}$ denote the population\textendash level \method{} classifier (Algorithm~\ref{alg:conftriage}) at threshold $\thresh \in [0, 1/2]$. For any predictor $\probhat: \Xcal \to [0, 1]$, any specialist backstop $\specdl$, and any $\thresh$,
\begin{equation}
\Risk(\predY_{\thresh}) - \Risk(\predY^*_{\thresh}) \leq 2\, \Delta(\probhat) + \pi_{\Delta}(\thresh).
\label{eq:oracle_bound}
\end{equation}
\end{theorem}

\begin{remark}[Tsybakov\textendash margin corollary]
\label{rem:tsybakov}
Under the standard Tsybakov\textendash margin condition with parameter $\alpha \in [0, \infty]$ and constant $C$ near each deferral boundary, namely $\Pr_{\Pdist}\bigl[\bigl|\eta(\xfeat) - 1/2 - \thresh\bigr| \leq t \text{ or } \bigl|\eta(\xfeat) - 1/2 + \thresh\bigr| \leq t\bigr] \leq 2 C t^\alpha$ for all $t \geq 0$, the deferral discrepancy is controlled by the calibration error: $\pi_{\Delta}(\thresh) \leq 2 C \mathbb{E}_{\Pdist}|\probhat - \eta|^\alpha$. Specializing to $\alpha = 1$,
\begin{equation*}
\Risk(\predY_{\thresh}) - \Risk(\predY^*_{\thresh}) \leq 2(1 + C)\, \Delta(\probhat),
\end{equation*}
i.e., the excess risk is first\textendash order in the $L^1$ calibration error.
\end{remark}

\begin{remark}[Why this is the right comparison]
The Bayes\textendash optimal benchmark in~\eqref{eq:bayes_oracle} uses the \emph{same} backstop $\specdl$ that \method{} uses. We therefore measure how well \method{} exploits the LLM's predictive content, separating it from the question of how good the backstop is. A weaker baseline (the Bayes classifier with no deferral) would conflate the two. A stronger baseline (the Bayes classifier with the Bayes\textendash optimal backstop) would not give actionable guidance because that backstop is not implementable.
\end{remark}

\begin{remark}[Connection to ECE]
The unobservable quantity $\Delta(\probhat)$ is bounded above, under mild regularity on $\eta$, by the ECE of $\probhat$ plus a binning approximation term that vanishes as the bin width $\to 0$. Theorem~\ref{thm:oracle} therefore translates directly into a bound expressed in terms of the ECE, which we measure empirically. The Platt\textendash scaling step in Algorithm~\ref{alg:conftriage} is the operational lever that drives ECE down on $\Dcalib$ and, through Theorem~\ref{thm:oracle}, drives the excess risk of \method{} down on $\Dtest$.
\end{remark}

\begin{proof}[Proof sketch of Theorem~\ref{thm:oracle}]

We decompose pointwise excess risk by the four cases of $(\hat{\defind}, \defind^*) \in \{0, 1\}^2$. When both predict ($\hat{\defind} = \defind^* = 0$), the pointwise excess risk is at most $|2\eta - 1| \indicator\{\hat{g} \neq g^*\}$, which is bounded by $2|\probhat - \eta|$ by the classical plug\textendash in identity (when the two classifiers disagree, $|\eta - 1/2| \leq |\probhat - \eta|$). When both defer, the excess is zero. When the deferral choices disagree, the pointwise excess is at most one. Taking expectations and noting that the disagreement set has measure $\pi_{\Delta}(\thresh)$ yields~\eqref{eq:oracle_bound}.
\end{proof}

\subsection{Summary of Theoretical Guarantees}

Theorems~\ref{thm:triage} and~\ref{thm:oracle} together provide a self\textendash contained statistical foundation for the proposed framework. Table~\ref{tab:theorem_summary} summarizes the theoretical guarantees, together with their assumptions and operational implications. Theorem~\ref{thm:triage} converts the empirical retained and deferred errors into a finite\textendash sample certificate of the combined risk. Theorem~\ref{thm:oracle} formally connects calibration to triage optimality: a well\textendash calibrated LLM probability suffices for near\textendash Bayes\textendash optimal deferral.

\begin{table*}[tb]
\centering
\caption{Summary of the theoretical guarantees. Each theorem links an observable quantity to a deployment-relevant property.}
\label{tab:theorem_summary}
\setlength{\tabcolsep}{6pt}
\renewcommand{\arraystretch}{1.15}
\small
\begin{tabular}{@{}p{0.105\textwidth} p{0.235\textwidth} p{0.165\textwidth} p{0.185\textwidth} p{0.215\textwidth}@{}}
\toprule
\textbf{Theorem} &
\textbf{What it bounds} &
\textbf{Assumption} &
\textbf{Typical slack} &
\textbf{Operational implication} \\
\midrule

Thm.~\ref{thm:triage}
\newline (Combined risk)
&
Population error
$\Risk(\predY)\leq$
empirical retained + deferred error + slack
&
i.i.d.\ held--out test fold
&
$\sqrt{\log(1/\delta)/(2n)}$;
$\approx0.0886$
at $n=191$, $\delta=0.05$
&
Per-fold finite-sample certificate of combined error for any threshold $\thresh$. \\

Thm.~\ref{thm:oracle}
\newline (Oracle inequality)
&
$\Risk(\predY)-\Risk^*
\leq
2\Delta+\pi_\Delta$,
with
$\Delta=\mathbb{E}|\probhat-\eta|$
&
None beyond integrability
(Tsybakov margin tightens to
$\mathcal{O}(\Delta)$)
&
Margin-dependent;
first-order in $\Delta$
under linear margin
&
Calibration (measurable via ECE) directly controls the gap to Bayes-optimal deferral. \\

\bottomrule
\end{tabular}
\end{table*}

\section{Datasets and Evaluation Protocol}
\label{sec:datasets}

\subsection{Dataset}

The LIDC\textendash IDRI \cite{armato2011lung} is the primary benchmark that was utilized for conducting the research. The full database contains 1018 thoracic CT scans with lesions marked as a nodule by at least one thoracic radiologist. Following standard practice \cite{xie2018knowledge, wu2023self, shi2021semi}, we exclude nodules with diameter $\leq$ 3\,mm, and we exclude nodules with median malignancy rating equal to 3 because of irreducible diagnostic ambiguity. After these inclusion criteria, 955 nodules are retained (427 benign, 528 malignant), partitioned into five patient\textendash wise cross\textendash validation folds.

We partitioned the labeled dataset $\Dcal = \{(\xfeat_i, \labelvar_i)\}_{i=1}^N$ using patient-level five-fold cross-validation to prevent patient-level leakage. For each iteration, fold $i$ was used as the held-out test set $\Dtest$, while fold $(i+1)\bmod 5$ served as the calibration set $\Dcalib$ for fitting the Platt-scaling parameters and determining the operating threshold. The learned calibration model was then applied only to the corresponding test fold, ensuring that each sample received exactly one out-of-fold calibrated probability.

\subsection{Reporting Protocol}
\label{sec:reporting}

For all models and ablation experiments, we report accuracy (Acc.), F1 score, and AUC. In addition, the expected calibration error (ECE), Brier score, sensitivity (Sen.), specificity (Spe.), and Cohen's kappa ($\kappa$) are reported where relevant. All point estimates are accompanied by 95\% percentile bootstrap confidence intervals computed from 300 patient-level bootstrap resamples of $\Dtest$, where resampling with replacement at the patient level preserves the dependence structure introduced by multiple nodules per patient.

Pairwise testing under FDR control: Pairwise model comparisons under a fixed ablation use the paired bootstrap on AUC at the patient level; the resulting $p$\textendash values are subjected to Benjamini\textendash Hochberg false\textendash discovery\textendash rate correction \cite{benjamini1995controlling} across all reported comparisons. Tables~\ref{tab:pairwise_auc} use 300 bootstrap resamples for these pairwise tests. Because this bootstrap budget implies a minimum attainable raw $p$-value of approximately $1/(300+1)\approx 0.0033$, we report significance as $^{**}$ for $q < 0.01$ and $^{*}$ for $q < 0.05$.

\section{Results and Analysis}
\label{sec:results}

This section is organized to match the contributions: Section~\ref{sec:res_main} reports the main triage table comparing \method{} against specialized DL; Section~\ref{sec:res_ablation} establishes the empirical discovery that linguistic descriptions dominate the LLM signal; Section~\ref{res:operating_threshold} reports calibration and triage operating curves and instantiates the bounds of Theorems~\ref{thm:triage}; Section~\ref{sec:reader_variability} places \method{}'s combined performance in the context of the LIDC\textendash IDRI reference\textendash reader annotations using public data only.

\subsection{Main Comparison and \method{} Performance}
\label{sec:res_main}

Table \ref{tab:comparison} presents the performance comparison among the evaluated LLMs, the DL backstop, and the proposed \method{}. The best-performing LLM results are underlined, while the overall best results across the entire table are highlighted in bold.

\subsubsection{LLM vs \method{} and Backstop Performance}

For the LLMs, Table \ref{tab:comparison} reveals two key observations. First, two of the five evaluated LLMs---Gemini\textendash 3.1\textendash Flash\textendash Lite (AUC = $0.907$) and Mistral\textendash Large (AUC = $0.901$)---achieved an AUC above 0.9 without any image\textendash level training. This finding supports the operational hypothesis that, given a faithful language representation of structured radiological attributes, an LLM can attain discriminative performance comparable to specialist DL models without direct image-based training. Second, two high-performing models based on AUC (Mistral\textendash Large and Qwen\textendash 2.5\textendash 72B) are open-weight models, reducing vendor lock-in concerns that often limit hospital deployment.

The proposed \method{} achieved the highest F1 score ($0.882$), while the DL backstop (Certain-Net) attained the highest AUC ($0.933$). Notably, a recent longitudinal LLM-based study reported an accuracy of $0.88$, which is comparable to the accuracy achieved by \method{}. The primary motivation behind \method{} is that 76.5\% of cases can be handled without requiring any image-trained model, while still achieving the best overall F1 score.

\begin{table}[tb]
\centering
\caption{Performance comparison among the LLM, DL backstop, and the proposed \method{}. The language-only LLM achieved performance comparable to that of the specialist DL backstop.}
\label{tab:comparison}
\setlength{\tabcolsep}{4pt}
\renewcommand{\arraystretch}{1.05}
\small
\begin{tabular}{@{}l c c c c@{}}
\toprule
\textbf{System} & \textbf{Acc.} & \textbf{F1} & \textbf{AUC} & \textbf{AP} \\
\midrule
\multicolumn{5}{@{}l}{\emph{LLM (zero\textendash shot, language\textendash only)}} \\
Gemini\textendash 3.1\textendash Flash      & \underline{0.830} & 0.844 & \underline{0.907} & \underline{0.894} \\
Mistral\textendash Large                    & 0.822 & 0.829 & 0.901 & 0.892 \\
Qwen\textendash 2.5\textendash 72B          & 0.827 & \underline{0.857} & 0.880 & 0.849 \\
Claude\textendash 3.5\textendash Haiku      & 0.736 & 0.775 & 0.798 & 0.790 \\
GPT\textendash 4o\textendash mini            & 0.731 & 0.719 & 0.796 & 0.798 \\
\midrule
\multicolumn{5}{@{}l}{\emph{Specialist DL backstop}} \\
Certain\textendash Net (EffNetV2B0)   & 0.851 & 0.864 & \textbf{0.933} & \textendash{} \\

\midrule
\multicolumn{5}{@{}l}{\emph{LLM with DL backstop (Calibrated)}} \\
\method{} & 0.869 & \textbf{0.882} & 0.920 & \textendash{} \\

\midrule
\multicolumn{5}{@{}l}{\emph{Reference (different task\textsuperscript{a}; orientation only)}} \\
Mao et~al.~\cite{mao2025assessments} (longitudinal) & \textbf{0.880} & \textendash{} & \textendash{} & \textendash{} \\

\bottomrule
\end{tabular}

\vspace{2mm}
\footnotesize
\textsuperscript{a}Mao et~al.~\cite{mao2025assessments} is on a different (longitudinal CT) task and included for orientation only.

\end{table}

\subsubsection{Closed versus Open\textendash Source Aggregate} Aggregating by family (closed: GPT\textendash 4o\textendash mini, Claude\textendash 3.5\textendash Haiku, Gemini\textendash 3.1\textendash Flash\textendash Lite; open\textendash source: Mistral\textendash Large, Qwen\textendash 2.5\textendash 72B), the family\textendash level aggregates on the language\textendash only regime are similar (mean AUC 0.851 for closed, 0.880 for open\textendash source; Fig.~\ref{fig:family_heatmap}). The single best configuration is Gemini\textendash 3.1\textendash Flash\textendash Lite (closed, AUC 0.907), and the second\textendash best is Mistral\textendash Large (open\textendash source, AUC 0.901); the spread within each family exceeds the across\textendash family aggregate gap. The operational implication is that Pillar~P1 (language as the modality) does not require closed\textendash vendor models, which supports the multi\textendash vendor portability claim of Section~\ref{sec:vendors}.

\begin{figure}[tb]
\centering
\includegraphics[width=1\columnwidth]{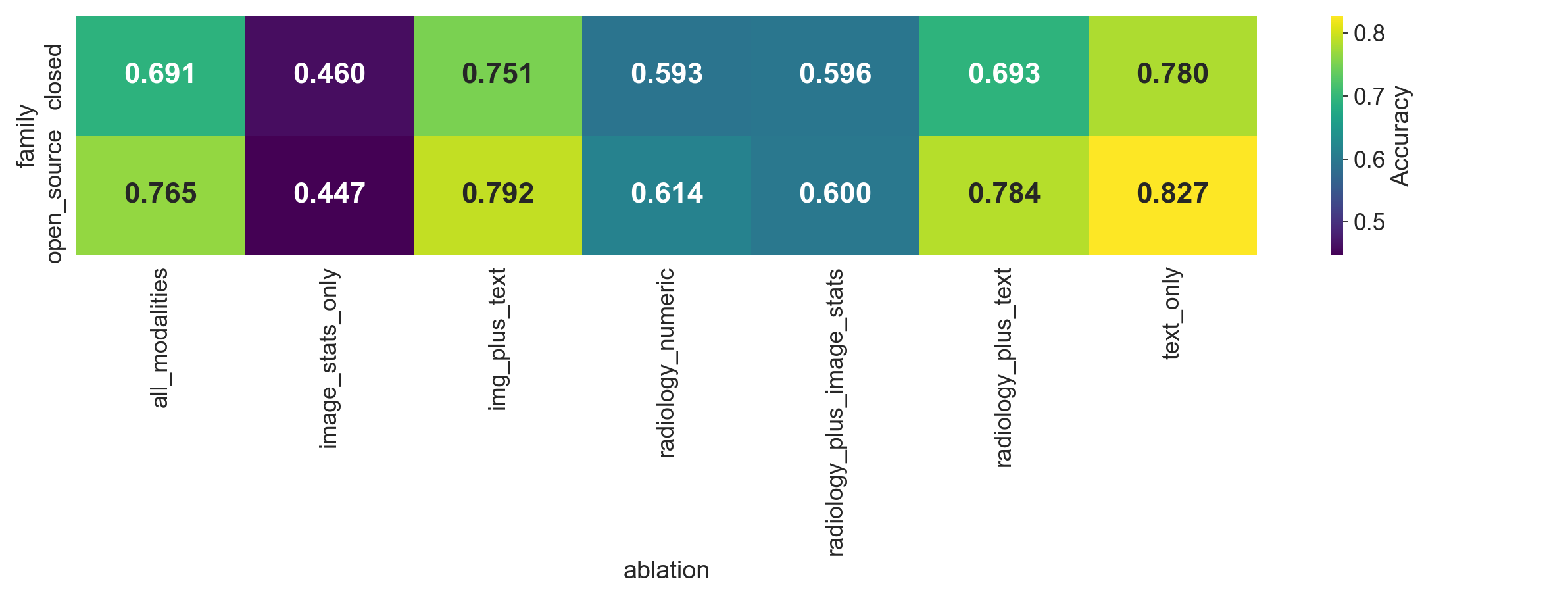}\\[2pt]
\includegraphics[width=0.95\columnwidth]{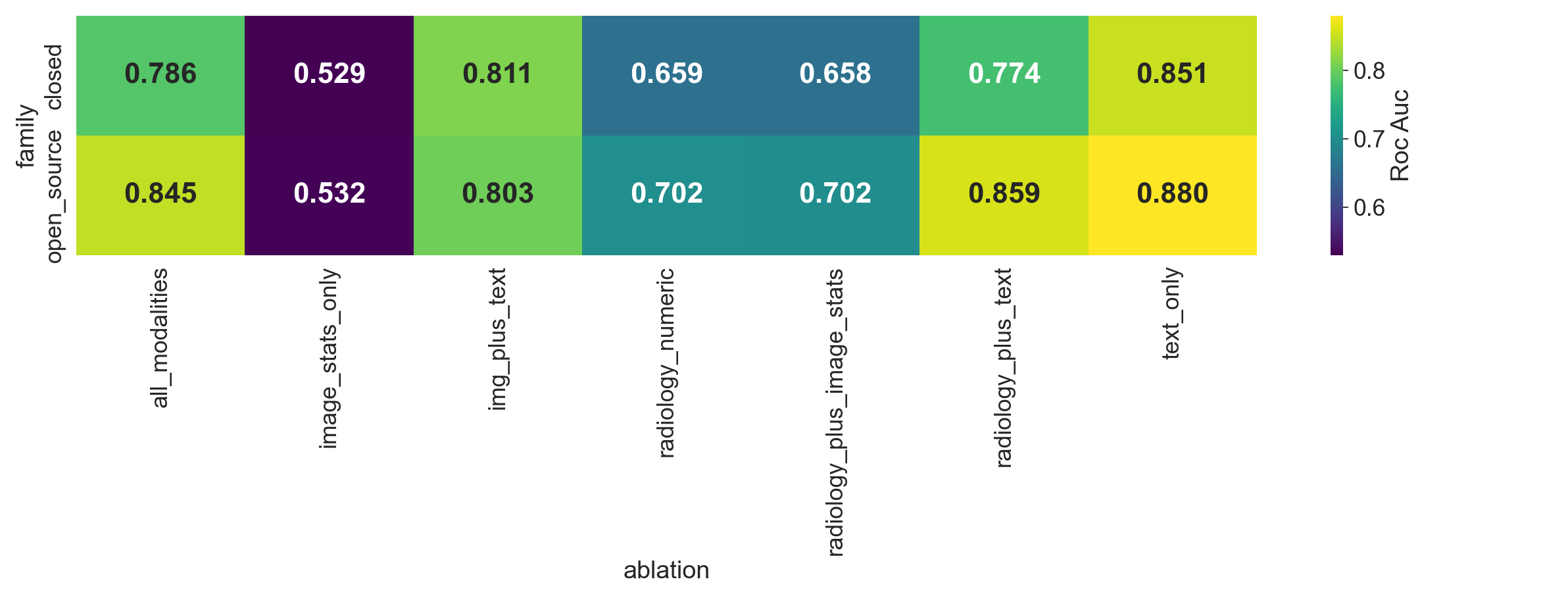}

\caption{Performance comparison between closed- and open-source LLMs across seven input regimes. Accuracy (top) and AUC (bottom) exhibit similar qualitative trends for both model families. Open-source models match or slightly outperform closed-source models in language-only and language-inclusive settings.}

\label{fig:family_heatmap}
\end{figure}

\begin{figure}[tb]
    \centering
    \includegraphics[width=1\linewidth]{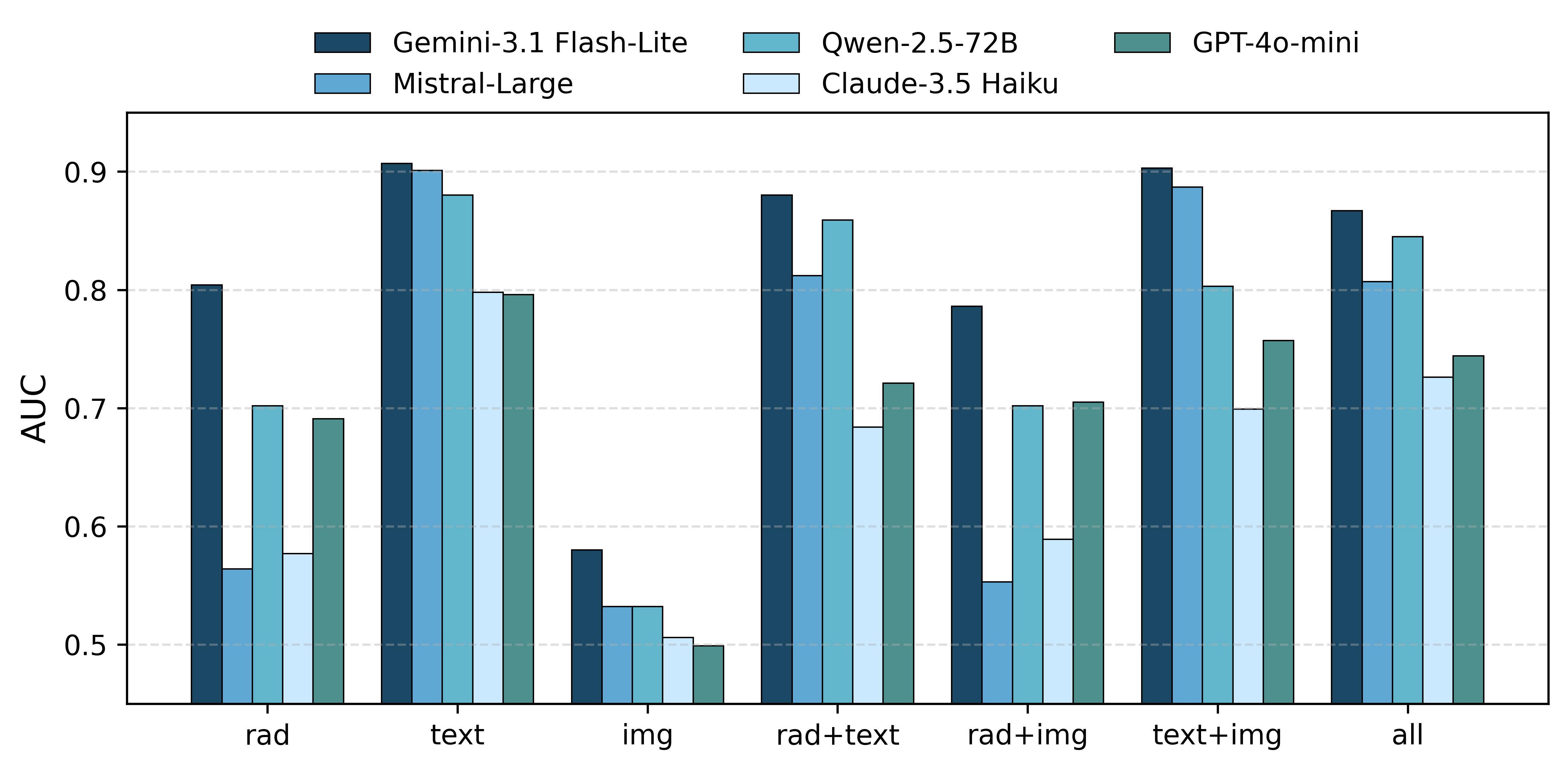}
    \caption{Gemini-3.1-Flash-Lite achieved the highest AUC among all LLMs across all seven input regimes. The text-only regime yielded the best performance, and language\textendash inclusive regimes consistently outperformed the remaining regimes.}
    \label{fig:AUC_comparison}
\end{figure}

\subsection{Empirical Discovery}
\label{sec:res_ablation}

\begin{table}[tb]
\centering
\caption{Per-model performance on LIDC\textendash IDRI ($n{=}955$ nodules) across seven input regimes. \textbf{Bold} = best per model; \underline{underline} = second-best. Best overall: Gemini-3.1-Flash-Lite + $\phimap_{\mathrm{text}}$ (AUC 0.907).}
\label{tab:main_results}
\setlength{\tabcolsep}{3.5pt}
\renewcommand{\arraystretch}{1.02}
\footnotesize
\begin{tabular}{@{}ll cccc@{}}
\toprule
\textbf{Model} & \textbf{Regime} & \textbf{Acc.} & \textbf{F1} & \textbf{AUC} & \textbf{AP} \\
\midrule
\multirow{7}{*}{\shortstack[l]{Gemini-3.1\\Flash-Lite}} & $\phimap_{\mathrm{rad}}$ & .644 & .551 & .804 & .817 \\
 & $\phimap_{\mathrm{text}}$ & \textbf{.830} & \textbf{.844} & \textbf{.907} & \textbf{.894} \\
 & $\phimap_{\mathrm{img}}$ & .447 & .000 & .580 & .622 \\
 & $\phimap_{\mathrm{rad+text}}$ & .703 & .656 & .880 & .872 \\
 & $\phimap_{\mathrm{rad+img}}$ & .644 & .551 & .786 & .806 \\
 & $\phimap_{\mathrm{text+img}}$ & \underline{.819} & \underline{.831} & \underline{.903} & \underline{.893} \\
 & $\phimap_{\mathrm{all}}$ & .679 & .615 & .867 & .858 \\
\cmidrule(lr){1-6}

\multirow{7}{*}{Mistral-Large} & $\phimap_{\mathrm{rad}}$ & .518 & .589 & .564 & .614 \\
 & $\phimap_{\mathrm{text}}$ & \underline{.822} & \underline{.829} & \textbf{.901} & \textbf{.892} \\
 & $\phimap_{\mathrm{img}}$ & .444 & .007 & .532 & .567 \\
 & $\phimap_{\mathrm{rad+text}}$ & .755 & .756 & .812 & .805 \\
 & $\phimap_{\mathrm{rad+img}}$ & .521 & .604 & .553 & .594 \\
 & $\phimap_{\mathrm{text+img}}$ & \textbf{.839} & \textbf{.854} & \underline{.887} & \underline{.874} \\
 & $\phimap_{\mathrm{all}}$ & .754 & .752 & .807 & .796 \\
\cmidrule(lr){1-6}

\multirow{7}{*}{Qwen-2.5-72B} & $\phimap_{\mathrm{rad}}$ & .614 & .499 & .702 & .742 \\
 & $\phimap_{\mathrm{text}}$ & \textbf{.827} & \textbf{.857} & \textbf{.880} & \textbf{.849} \\
 & $\phimap_{\mathrm{img}}$ & .447 & .000 & .532 & .570 \\
 & $\phimap_{\mathrm{rad+text}}$ & .784 & .783 & \underline{.859} & \underline{.844} \\
 & $\phimap_{\mathrm{rad+img}}$ & .600 & .466 & .702 & .725 \\
 & $\phimap_{\mathrm{text+img}}$ & \underline{.792} & \underline{.838} & .803 & .762 \\
 & $\phimap_{\mathrm{all}}$ & .765 & .755 & .845 & .835 \\
\cmidrule(lr){1-6}

\multirow{7}{*}{\shortstack[l]{Claude-3.5\\Haiku}} & $\phimap_{\mathrm{rad}}$ & .573 & .719 & .577 & .596 \\
 & $\phimap_{\mathrm{text}}$ & \textbf{.736} & \textbf{.775} & \textbf{.798} & \textbf{.790} \\
 & $\phimap_{\mathrm{img}}$ & .503 & .531 & .506 & .554 \\
 & $\phimap_{\mathrm{rad+text}}$ & .648 & .753 & .684 & .677 \\
 & $\phimap_{\mathrm{rad+img}}$ & .562 & .713 & .589 & .610 \\
 & $\phimap_{\mathrm{text+img}}$ & .630 & .718 & .699 & \underline{.720} \\
 & $\phimap_{\mathrm{all}}$ & \underline{.669} & \underline{.756} & \underline{.726} & .718 \\
\cmidrule(lr){1-6}

\multirow{7}{*}{GPT-4o-mini} & $\phimap_{\mathrm{rad}}$ & .639 & \textbf{.740} & .691 & .682 \\
 & $\phimap_{\mathrm{text}}$ & \textbf{.731} & \underline{.719} & \textbf{.796} & \textbf{.798} \\
 & $\phimap_{\mathrm{img}}$ & .447 & .000 & .499 & .556 \\
 & $\phimap_{\mathrm{rad+text}}$ & .667 & .674 & .721 & .721 \\
 & $\phimap_{\mathrm{rad+img}}$ & .655 & .718 & .705 & .706 \\
 & $\phimap_{\mathrm{text+img}}$ & \underline{.715} & .709 & \underline{.757} & \underline{.757} \\
 & $\phimap_{\mathrm{all}}$ & .664 & .665 & .744 & .749 \\
\bottomrule
\end{tabular}
\end{table}


\subsubsection{Language Dominates the Signal in LLMs}

Across all five LLMs applied, language\textendash only input ($\phimap_{\mathrm{text}}$) achieves either the top or second\textendash best AUC as shown in Table \ref{tab:main_results}. Fig. \ref{fig:AUC_comparison} indicates that the best overall configuration is Gemini\textendash 3.1\textendash Flash\textendash Lite + $\phimap_{\mathrm{text}}$ at AUC = 0.907, followed by Mistral\textendash Large + $\phimap_{\mathrm{text}}$ at AUC = 0.901 and Qwen\textendash 2.5\textendash 72B + $\phimap_{\mathrm{text}}$ at AUC = 0.880. The image\textendash statistics regime, in contrast, collapses to AUC in $[0.499, 0.580]$, with several models defaulting to a single class (F1 = 0.000) when the input contains no diagnostic signal: a clean negative\textendash control outcome.

The mean text-vs-image AUC gap remains large ($\Delta \approx 0.327$; per-model gaps: Gemini $+0.327$, Mistral $+0.369$, Qwen $+0.348$, Claude $+0.292$, GPT-4o-mini $+0.297$). Pairwise paired-bootstrap tests further show that stronger models significantly outperform weaker ones under matched ablations: under $\phimap_{\mathrm{all}}$, Mistral, Qwen, and Gemini each exceed GPT-4o-mini by $\Delta$AUC in $[0.132,\,0.144]$ with $q<0.01$ (Table~\ref{tab:pairwise_auc}). In contrast, pairwise differences among the top three models under $\phimap_{\mathrm{all}}$ remain within the bootstrap noise band and are not significant after FDR correction.

\subsubsection{Gemini's Anomaly on Numeric Input}
Under the radiology-numeric regime ($\phimap_{\mathrm{rad}}$, structured attributes without language rendering), Gemini-3.1-Flash-Lite still exhibits a clear separation from the other models (Table~\ref{tab:pairwise_auc}). A plausible interpretation is that Gemini is comparatively better at extracting signal directly from serialized numeric fields. While the mechanism cannot be verified externally, the empirical gap is robust in the updated full-cohort analysis. Substantively, this suggests that Pillar~P1 (language as modality) remains highly effective overall, but is not strictly necessary for all vendor/model families.

\begin{table}[tb]
\centering
\small
\setlength{\tabcolsep}{3pt}
\caption{Strong--versus--weak-model gaps after FDR correction. Representative pairwise AUC comparisons on LIDC\textendash IDRI using paired bootstrap with Benjamini--Hochberg FDR correction. $\Delta$ denotes AUC difference (Model~A $-$ Model~B). Significance markers: $^{**}{\,}q<0.01$, $^{*}{\,}q<0.05$.}
\label{tab:pairwise_auc}
\begin{tabular}{@{}l c c c@{}}
\toprule
\textbf{A vs.\ B} & $\Delta$\textbf{AUC} & \textbf{95\% CI} & $\boldsymbol{q}$ \\
\midrule
\multicolumn{4}{@{}l}{\emph{All modalities} ($\phimap_{\mathrm{all}}$)} \\
Mistral vs.\ GPT & $+0.062$ & $[0.038, 0.090]$ & $^{**}$ \\
Qwen vs.\ GPT & $+0.101$ & $[0.075, 0.132]$ & $^{**}$ \\
Gemini vs.\ GPT & $+0.123$ & $[0.096, 0.148]$ & $^{**}$ \\
Mistral vs.\ Qwen & $-0.039$ & $[-0.055, -0.024]$ & $^{**}$ \\
\midrule
\multicolumn{4}{@{}l}{\emph{Numeric attributes only} ($\phimap_{\mathrm{rad}}$)} \\
Gemini vs.\ Mistral & $+0.240$ & $[0.210, 0.268]$ & $^{**}$ \\
Gemini vs.\ GPT & $+0.113$ & $[0.079, 0.148]$ & $^{**}$ \\
Gemini vs.\ Qwen & $+0.102$ & $[0.078, 0.129]$ & $^{**}$ \\
\midrule
\multicolumn{4}{@{}l}{\emph{Image statistics only} ($\phimap_{\mathrm{img}}$)} \\
Gemini vs.\ GPT & $+0.080$ & $[0.042, 0.114]$ & $^{**}$ \\
Mistral vs.\ GPT & $+0.033$ & $[-0.011, 0.069]$ & 0.22 \\
\bottomrule
\end{tabular}
\end{table}

\subsection{Operating Threshold Selection and Empirical Insights}
\label{res:operating_threshold}

\begin{figure}[tb]
    \centering
    \includegraphics[width=0.95\linewidth]{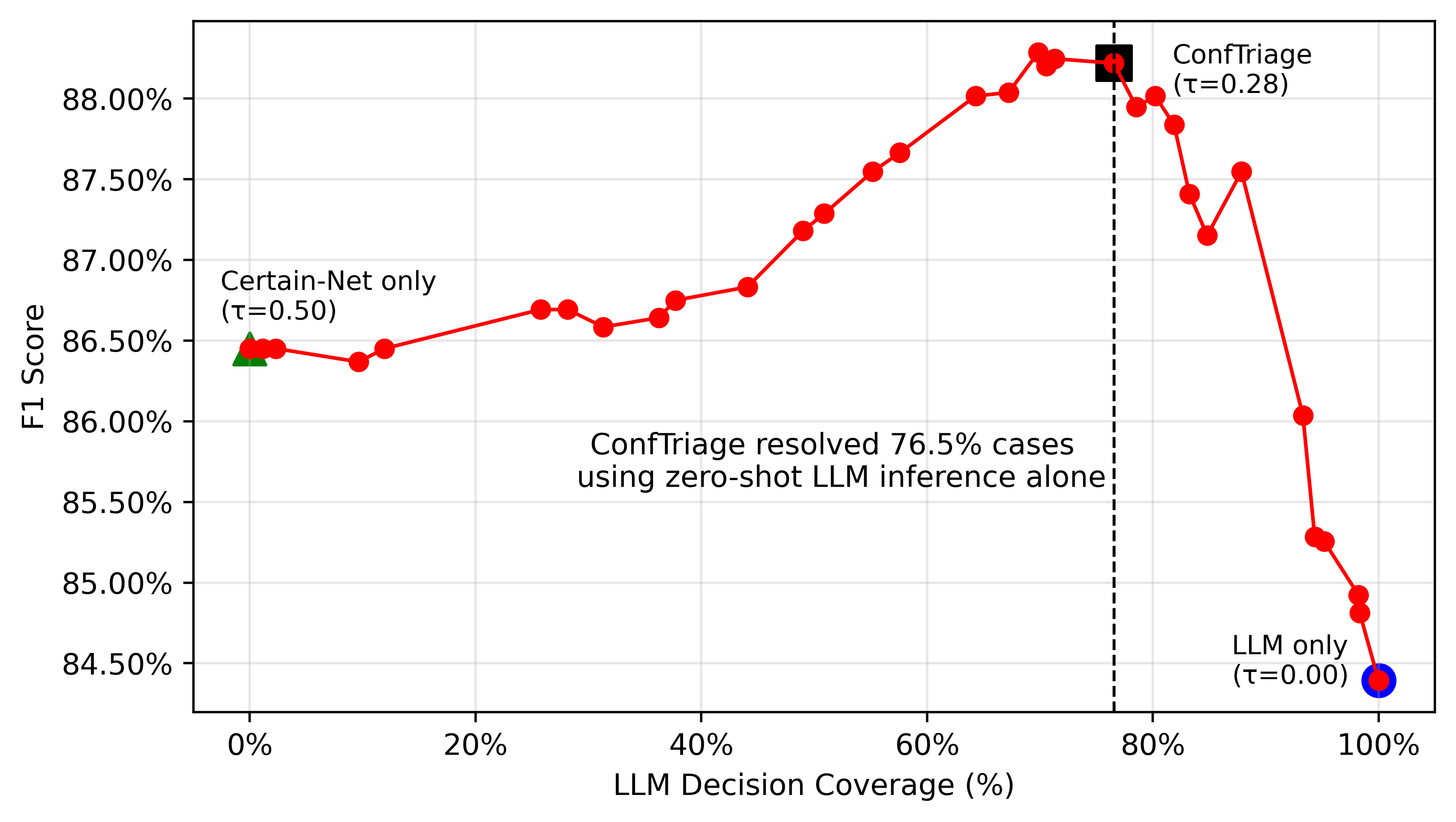}
    \caption{Elbow-point selection of the operating threshold ($\tau=0.28$), balancing predictive performance and LLM decision coverage.}
    \label{fig:conftriage_threshold_selection}
\end{figure}

The deferral threshold $\tau$ controls the operating tradeoff of ConfTriage. Lower values of $\tau$ result in reduced performance, whereas higher values of $\tau$ reduce coverage and increase reliance on the image-based DL specialist backstop. To identify a practically useful operating point, we aim to determine the minimum $\tau$ that provides high coverage while ensuring that performance gains obtained from additional referrals become marginal beyond that point. Therefore, we evaluated ConfTriage over threshold values ranging from 0 to 0.5 with a step size of 0.01. Fig.~\ref{fig:conftriage_threshold_selection} presents the relationship between LLM decision coverage and F1 score. Using the elbow method, the operating threshold was determined as $\tau = 0.28$. At $\tau = 0.28$, ConfTriage retained 76.5\% of cases at the LLM stage while referring only the remaining 23.5\% of uncertain cases to the specialist DL backstop. Consequently, more than three quarters of all diagnostic decisions were completed using zero-shot LLM inference alone, without requiring invocation of a task-specific image model.

Fig.~\ref{fig:conftriage_threshold_selection} also highlights that relying exclusively on the LLM ($\tau = 0.0$) or exclusively on the DL backstop ($\tau = 0.5$) is suboptimal, as both configurations result in lower F1 scores. In contrast, ConfTriage achieves superior performance through selective referral, leveraging the complementary strengths of the LLM and the DL backstop while maintaining high LLM coverage.



\subsection{Empirical Verification of Theorem~\ref{thm:triage}}

Theorem~\ref{thm:triage} was evaluated independently on each held--out cross--validation fold using the common operating threshold $\tau=0.28$. Each fold contained $n=191$ test samples, yielding a Hoeffding slack of
\[
\sqrt{\frac{\log(20)}{2\times191}}
=0.0886.
\]
Across the five folds, the empirical combined error was
$13.1\%\pm2.3\%$, while the resulting certified upper bound was
$21.9\%\pm2.3\%$. The worst fold attained an empirical error of
$16.8\%$, corresponding to a certified upper bound of $25.6\%$.
Thus, with probability at least $95\%$, the population error of ConfTriage on every held--out fold is bounded by its corresponding certificate, with the worst--case deployment certificate equal to $25.6\%$.

\subsection{Empirical Verification of Theorem~\ref{thm:oracle}}
\label{emperical_oracle}

\begin{table}[tb]
\centering
\caption{Empirical verification of Theorem~\ref{thm:oracle} through calibration-aware routing. Calibration substantially reduces ECE while increasing LLM decision coverage and slightly improving overall ConfTriage performance.}
\label{tab:theorem3_verification}
\setlength{\tabcolsep}{4pt}
\footnotesize
\begin{tabular}{lcccc}
\toprule
\textbf{Probability source} & \textbf{ECE\,$\downarrow$} & \textbf{Brier\,$\downarrow$} & \textbf{Cov.\,$\uparrow$} & \textbf{F1\,$\uparrow$}\\
\midrule
Raw verbalized confidence & 0.0845 & 0.1239 & 65.45\% & 0.8781\\
Platt-scaled confidence   & 0.0421 & 0.1230 & 76.54\% & 0.8822\\
\midrule
Improvement & $-50.2\%$ & $-0.7\%$ & $+11.1$~pp & $+0.46\%$\\
\bottomrule
\end{tabular}

\vspace{2pt}
{\scriptsize \emph{Note:} ECE, Brier, and F1 improvements are relative changes; the coverage (Cov.) improvement is an absolute change in percentage points (pp).}
\end{table}

To empirically assess the implications of Theorem~\ref{thm:oracle}, we compared ConfTriage using raw verbalized confidence scores against the proposed Platt-calibrated probabilities. The results are summarized in Table~\ref{tab:theorem3_verification}. Calibration reduced the expected calibration error (ECE) from 0.0845 to 0.0421, corresponding to an approximately 50\% relative reduction, while slightly improving the Brier score from 0.1239 to 0.1230. At the selected operating threshold ($\tau=0.28$), calibration also increased LLM decision coverage from 65.45\% to 76.54\% and improved the overall F1 score from 0.8781 to 0.8822. These results are consistent with Theorem~\ref{thm:oracle}, demonstrating that improved probability calibration leads to more effective routing decisions, allowing ConfTriage to defer fewer cases to the specialist backstop while maintaining slightly higher predictive performance.

\subsection{Comparison Against Reader Variability Using LOO Consensus}
\label{sec:reader_variability}

To assess whether ConfTriage operates within the range of expert-reader variability, we compared its predictions against the distribution of reader annotations available in LIDC\textendash IDRI. The LIDC\textendash IDRI dataset contains multiple independent radiologist assessments for each nodule; however, reader identities are not globally consistent across the dataset. In particular, the radiologist indexed as reader~1 for one nodule is not necessarily the same individual as reader~1 for another nodule. Consequently, reader-specific performance cannot be estimated across the entire dataset, and conventional inter-radiologist analyses assuming fixed reader identities are not applicable. To address this limitation, we adopted a LOO consensus protocol (described in Section~\ref{sec:LOO_protocol}), in which each reader annotation was compared against a consensus label derived from the remaining readers associated with the same nodule.

Table~\ref{tab:reader_comparison} summarizes the resulting comparison. Individual readers are evaluated against the LOO consensus using sensitivity, specificity, accuracy, AUC, and Cohen's $\kappa$. ConfTriage is evaluated at three operating thresholds ($t=0.3$, $0.5$, and $0.7$) to illustrate the sensitivity--specificity trade-off. Unlike AUC, which summarizes performance across all possible thresholds, varying $t$ allows direct comparison between model operating points and reader operating points. Fig.~\ref{fig:reader_roc} visualizes the same comparison in ROC space. Continuous curves correspond to ConfTriage, LLM, and Certain-Net, while points $R1$--$R4$ denote the operating points obtained by comparing each reader against the consensus of the remaining readers. Each reader therefore contributes a single sensitivity--specificity pair rather than a full ROC curve. Because the LOO protocol generates one comparison instance per reader annotation, model performance is evaluated on reader--nodule instances rather than unique nodules. Consequently, the reader-specific analyses involve 611, 587, 494, and 357 valid comparison instances for R1--R4, respectively, whereas ConfTriage, the LLM, and Certain-Net are evaluated across all 2049 valid reader--nodule instances generated by the LOO procedure.

\begin{table}[tb]
\centering
\caption{Comparison of individual readers, Certain-Net, LLM, and ConfTriage under the LOO consensus protocol.}
\label{tab:reader_comparison}
\resizebox{0.5\textwidth}{!}{
\begin{tabular}{lccccccc}
\toprule
\textbf{System} & $t$ & $n$ & \textbf{Sen.} & \textbf{Spe.} & \textbf{Acc.} & \textbf{AUC} & \textbf{Cohen's $\kappa$} \\

\midrule
R1 & -- & 611 & 0.820 & 0.891 & 0.851 & 0.910 & 0.701 \\
R2 & -- & 587 & 0.879 & 0.883 & 0.881 & 0.922 & 0.759 \\
R3 & -- & 494 & 0.908 & 0.853 & 0.885 & 0.918 & 0.764 \\
R4 & -- & 357 & 0.608 & 0.896 & 0.709 & 0.874 & 0.440 \\
\midrule
ConfTriage & 0.3 & 2049 & 0.967 & 0.801 & 0.898 & 0.952 & 0.785 \\
ConfTriage & 0.5 & 2049 & 0.942 & 0.878 & 0.915 & 0.952 & 0.825 \\
ConfTriage & 0.7 & 2049 & 0.884 & 0.914 & 0.897 & 0.952 & 0.790 \\
\midrule
Certain-Net & -- & 2049 & -- & -- & -- & 0.949 & -- \\
LLM & -- & 2049 & -- & -- & -- & 0.946 & -- \\
\bottomrule
\end{tabular}
}
\end{table}


\begin{figure}[tb]
\centering
\includegraphics[width=\linewidth]{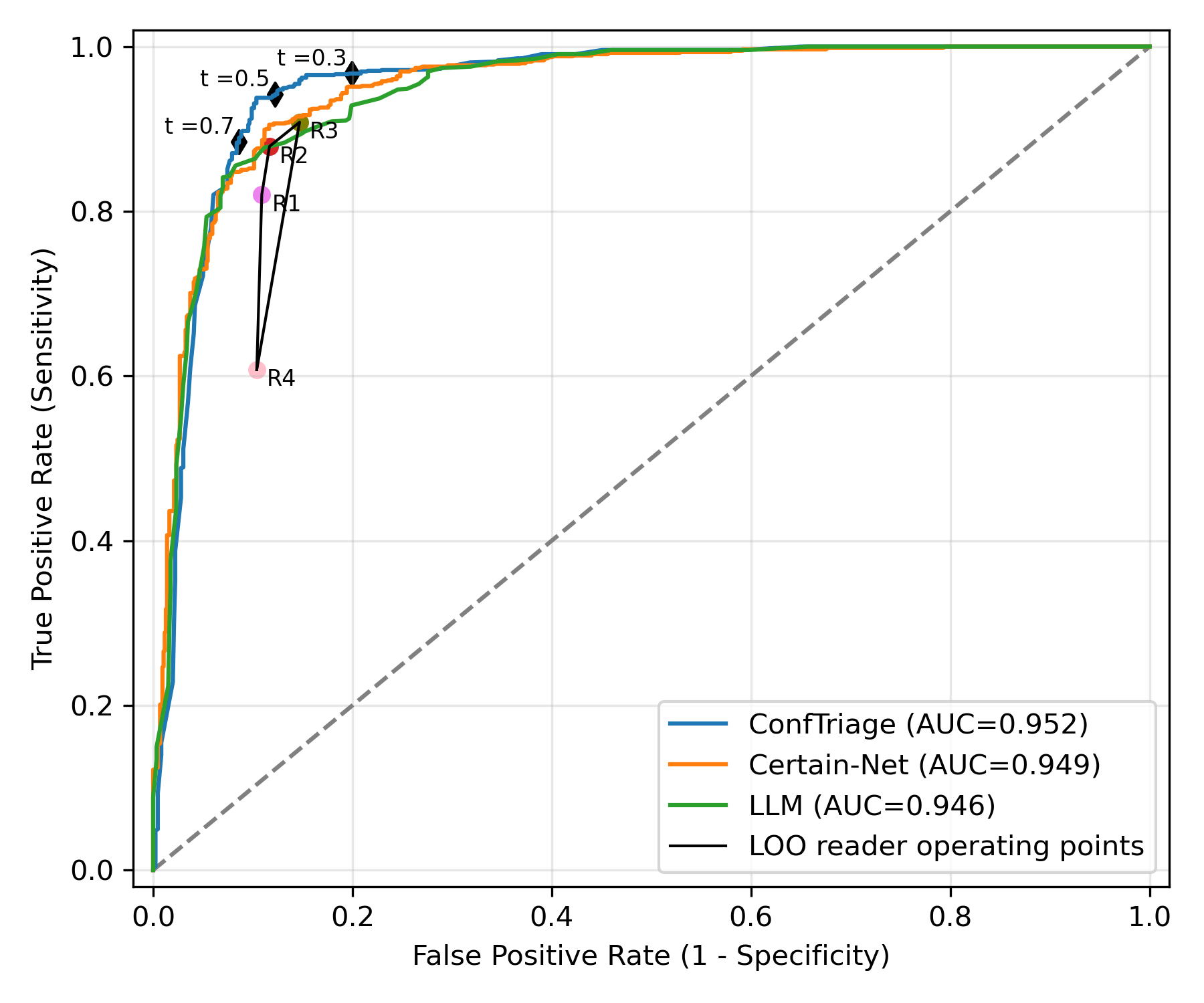}
\caption{ROC analysis using LOO consensus labels shows that ConfTriage achieves agreement with consensus that lies at the upper end of the range observed among individual reader operating points while slightly outperforming the specialist DL backstop (Certain-Net) and the standalone LLM prediction.}
\label{fig:reader_roc}
\end{figure}

Several observations emerge from this analysis. First, substantial variability exists among the reader operating points. Readers~R1--R3 exhibit relatively similar performance, achieving AUC values between 0.910 and 0.922 with Cohen's $\kappa$ values ranging from 0.701 to 0.764. In contrast, reader~R4 shows noticeably lower agreement with the remaining readers (AUC = 0.874, $\kappa=0.440$), highlighting the degree of observer variability present in the dataset.

Second, ConfTriage achieves an AUC of 0.952, which is higher than the AUC values observed for the individual reader operating points under the LOO-consensus framework and slightly outperforms both LLM (AUC = 0.946) and Certain-Net (AUC = 0.949). At the default threshold $t=0.5$, ConfTriage attains a sensitivity of 0.942, specificity of 0.878, accuracy of 0.915, and Cohen's $\kappa$ of 0.825, which are higher than the corresponding values observed for the individual reader operating points under the LOO-consensus framework.

Third, the ROC curves of ConfTriage, LLM, and Certain-Net occupy a performance region above the individual reader operating points, indicating that the proposed framework preserves the discriminative information available in both the specialist model and the underlying LLM while occupying a region of ROC space that is above the operating points observed for the individual readers.



Overall, this analysis quantifies the performance of ConfTriage relative to the variability inherent in the available expert-reader annotations. Under this framework, ConfTriage demonstrates agreement with consensus that is comparable to that observed for the individual reader operating points in the dataset while providing calibrated probabilistic outputs suitable for downstream triage and selective deferral.

\subsection{Ablation Studies}
\subsubsection{Synthetic Corruption Controls}
Three controls test the hypothesis that the LLM is exploiting genuine semantic content of the description rather than superficial language priors. 
(a)~Attributes negated by mapping each rating to its semantic opposite (``smooth'' $\to$ ``spiculated,'' etc.). 
(b)~Attributes paraphrased with neutral language while preserving meaning (``rounded'' instead of ``spherical'').
(c)~Attributes shuffled across nodules within the malignant class, preserving the marginal distribution of words while breaking the input-to-label relationship. The results we found accross 5 LLMs are presented in Table \ref{tab:text_corruption_ablation}.
Across models, performance under corruption generally decreases relative to the text-only condition, though the magnitude of degradation varies by corruption type and model.

Results in Table \ref{tab:text_corruption_ablation} show that the paraphrased descriptions typically preserve most of the original performance, indicating that models are not overly sensitive to surface-level lexical variation when semantic content is retained. 
In contrast, negated attributes lead to more substantial drops in AUC for several models (Claude, Gemini, GPT, and Qwen), suggesting that models are responsive to semantic polarity and that label-relevant features are being used. The shuffled control produces mixed effects. Overall, these results support the hypothesis that LLMs leverage meaningful semantic information contained in the radiological descriptions rather than relying solely on superficial lexical cues.

\begin{table}[tb]
\centering
\small
\caption{Text corruption ablation across models. AUC is reported with bootstrap 95\% confidence intervals. The text-only condition serves as the reference.}
\label{tab:text_corruption_ablation}
\setlength{\tabcolsep}{4pt}
\begin{tabular}{llc}
\toprule
\textbf{Model} & \textbf{Ablation} & \textbf{AUC [95\% CI]} \\
\midrule

Claude & Negated attrs     & 0.701 [0.631, 0.768] \\
                 & Paraphrase        & 0.772 [0.711, 0.829] \\
                 & Shuffle malignant & 0.711 [0.635, 0.787] \\
                 & Text only         & 0.772 [0.700, 0.841] \\

\midrule
Gemini & Negated attrs     & 0.820 [0.760, 0.876] \\
                      & Paraphrase        & 0.911 [0.867, 0.945] \\
                      & Shuffle malignant & 0.863 [0.809, 0.909] \\
                      & Text only         & 0.909 [0.865, 0.942] \\

\midrule
Mistral & Negated attrs     & 0.870 [0.817, 0.914] \\
               & Paraphrase        & 0.897 [0.844, 0.938] \\
               & Shuffle malignant & 0.891 [0.841, 0.932] \\
               & Text only         & 0.907 [0.861, 0.943] \\

\midrule
GPT & Negated attrs     & 0.615 [0.538, 0.688] \\
             & Paraphrase        & 0.762 [0.701, 0.827] \\
             & Shuffle malignant & 0.758 [0.704, 0.817] \\
             & Text only         & 0.790 [0.730, 0.851] \\

\midrule
Qwen & Negated attrs     & 0.756 [0.685, 0.813] \\
              & Paraphrase        & 0.828 [0.767, 0.878] \\
              & Shuffle malignant & 0.624 [0.580, 0.670] \\
              & Text only         & 0.870 [0.815, 0.918] \\

\bottomrule
\end{tabular}
\end{table}

\subsubsection{Description\textendash Source Ablation}
\label{app:descsource}

To investigate whether the LLM's predictive performance depends on the specific lexical realization of the input description or on the underlying radiological information itself, we compared the deterministic template (textual description in the dataset) 
against a paraphrase generated by Llama~3.1 8B Instruct and an adversarial paraphrase that preserves all attribute information while replacing potentially salient trigger terms. As shown in Table~\ref{tab:description_source_ablation}, the deterministic template achieved the highest performance (AUC = 0.907), while both the adversarial and Llama-generated paraphrase resulted in a modest AUC decrease. These findings suggest that the framework is relatively robust to changes in wording and does not rely heavily on specific trigger terms, while also indicating that the quality and structure of the generated descriptions influence predictive performance. 


\begin{table}[tb]
\centering
\caption{Description-source robustness analysis. All rows use the same target model and inference setting (single fixed run, temperature \(=0\)); only lexical realization of identical structured radiological attributes is varied (deterministic template, Llama paraphrase, adversarial paraphrase).}

\label{tab:description_source_ablation}
\begin{tabular}{lc}
\toprule
\textbf{Description source} & \textbf{AUC [95\% CI]} \\
\midrule
Deterministic template & 0.907 \,[0.861,\,0.943]  \\
Adversarial paraphrase & 0.867 \,[0.845,\,0.888] \\
Llama paraphrase       & 0.874 \,[0.851,\,0.894] \\
\bottomrule
\end{tabular}
\end{table}

\subsection{Prompt and Temperature Sensitivity}
\label{app:prompt}

Five prompt templates are tested, with results reported in Fig. \ref{fig:prompt_template_temperature_heatmap}: (T1)~zero\textendash shot direct (``Is this nodule benign or malignant?''); (T2)~thoracic\textendash radiologist persona; (T3)~chain\textendash of\textendash thought \cite{wei2022chain}; (T4)~three\textendash shot in\textendash context with class\textendash balanced examples drawn from $\Dcalib$; (T5)~Fleischner Society 2017 guidelines preamble. Five temperatures $T \in \{0.0, 0.3, 0.5, 0.7, 1.0\}$ are run with five replicates each on the best (model, $\phimap_{\mathrm{text}}$) configuration. Overall performance was strong and stable (AUC range $\approx 0.884$--$0.916$), with low replicate variability (typical $\pm 0.001$--$\pm 0.004$). The thoracic-radiologist persona template (T2) performed best overall, peaking at $T=0.5$ (AUC $\approx 0.916$). Temperature had a secondary effect relative to template choice, with mild degradation at higher temperature for some templates. Prompts are shared in Appendix \ref{app:prompt_ablation_examples}.

\begin{figure}[tb]
    \centering
    \includegraphics[width=\linewidth]{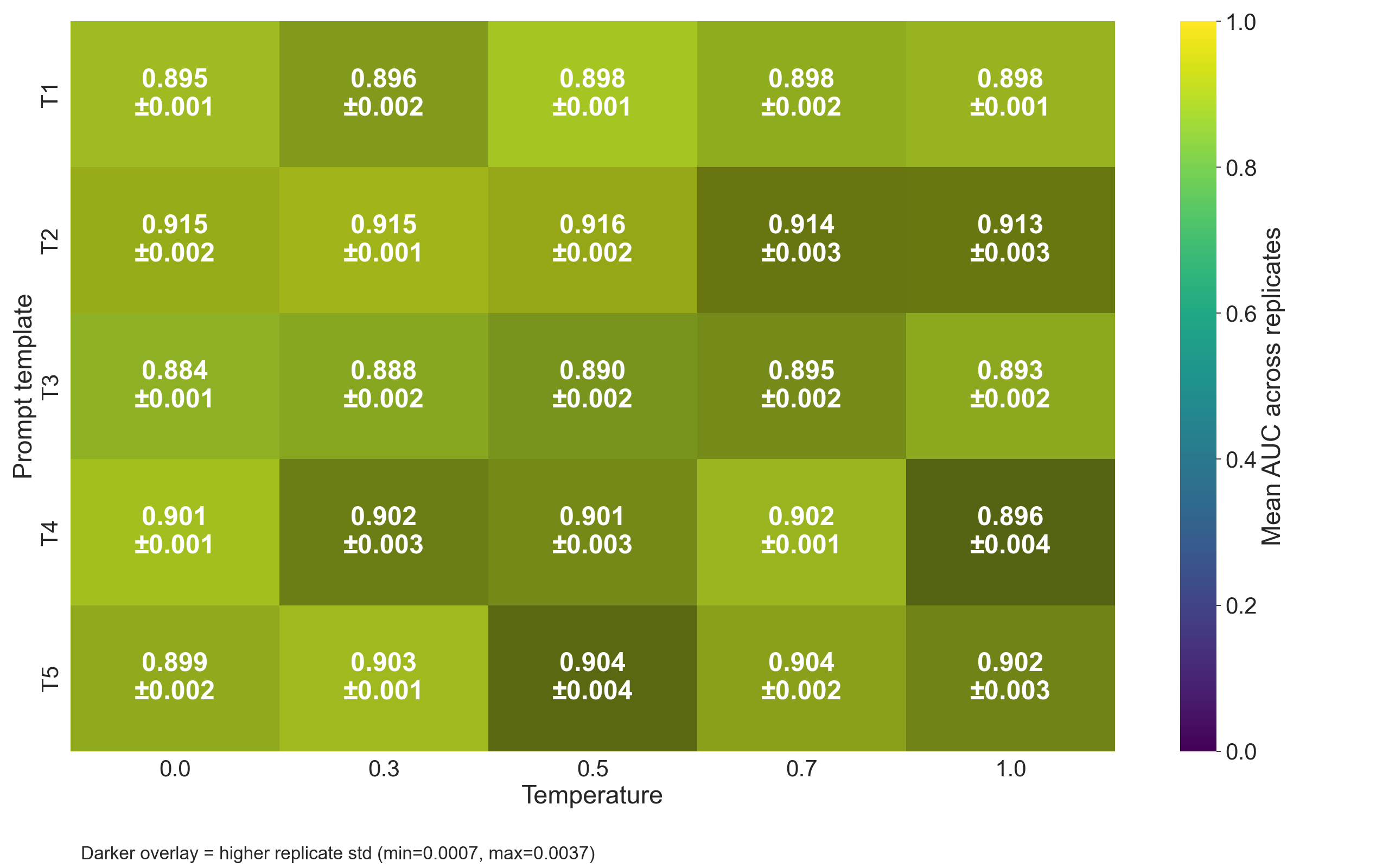}
    \caption{ConfTriage LLM prompt and temperature sensitivity analysis using text-only prompt ablation. Cell values report mean AUC ($\pm$SD) over five replicates, while darker cell shading indicates higher replicate variability.}

    \label{fig:prompt_template_temperature_heatmap}
\end{figure}

\subsection{Near\textendash Leak Attribute Removal}
\label{app:nearleak}

To investigate whether the LLM's predictive performance is primarily driven by a small number of highly malignancy-correlated attributes or by broader radiological information, we conducted an attribute-removal ablation study by excluding the LIDC\textendash IDRI spiculation and margin attributes individually and jointly from the input descriptions. As shown in Table~\ref{tab:nearleak-auc-drops}, removing \emph{spiculation} did \emph{not} reduce performance; instead, AUC increased from 0.907 (baseline) to 0.912 (\(\Delta\)AUC \(=+0.005\)). Removing \emph{margin} produced a near-null change (AUC 0.900; \(\Delta\)AUC \(=-0.007\)), and joint removal of both cues also yielded only a minimal shift (AUC 0.902; \(\Delta\)AUC \(=+0.005\)). Overall, these results do not support a dominant dependence on either spiculation or margin tokens alone; performance appears to be sustained by broader information in the description rather than a single near-leak attribute.

\begin{table}[tb]
\centering
\small

\caption{Near-leak attribute-removal ablation on deterministic attribute-based descriptions.}
\label{tab:nearleak-auc-drops}
\begin{tabular}{lcc}
\toprule
\textbf{Setting} & \textbf{AUC (95\% CI)} & \textbf{$\Delta$ vs baseline} \\
\midrule
baseline\_template & 0.907 (0.861, 0.943)  & -- \\
remove\_spic. & 0.912 (0.893, 0.932) & +0.005 \\
remove\_margin & 0.900 (0.880, 0.920)  & -0.007 \\
remove\_spic.\_margin & 0.902 (0.882, 0.919) & -0.005 \\
\bottomrule
\end{tabular}
\end{table}

\section{Discussion}
\label{sec:discussion}

\subsection{Operational Implication} The headline finding is not that LLMs ``beat'' specialist DL; the finding is that, when a structured radiology workflow already produces nodule attributes, an LLM operating on a faithful language rendering of those attributes produces predictions of clinically meaningful quality at zero training cost and modest API cost. \method{} extends this from a measurement to a methodological framework: the calibrated probability and the deferral threshold turn LLM confidence into a defensible routing decision, and Theorem~\ref{thm:triage} provides a finite\textendash sample guarantee on the resulting combined error. The specialist DL backstop's operating curve, together with the uncertainty distributions in Fig.~\ref{fig:uncertainty_distribution}, establishes that the backstop component does emit a calibrated rejection signal, which is the empirical precondition that the combined \method{} operating curve relies on. We position the present paper as a methodological contribution validated on public benchmarks; prospective clinical validation is the natural follow\textendash on study and is enabled by the public release.

\begin{figure}[tb]
\centering
\includegraphics[width=0.65\columnwidth]{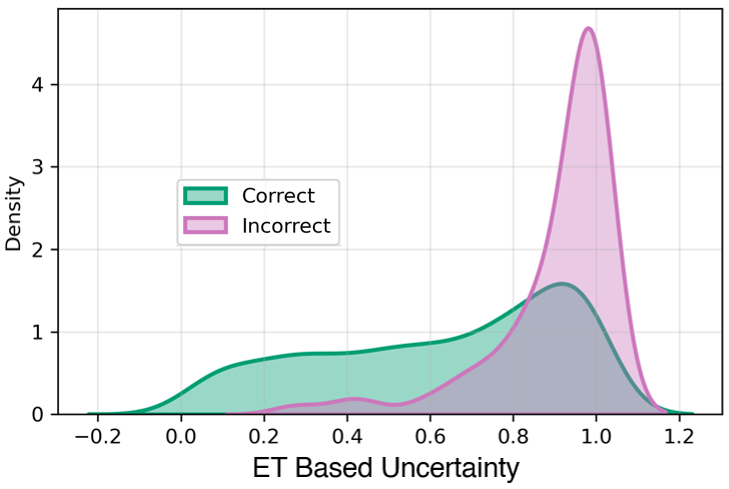}
\caption{Uncertainty distribution of the specialist DL backstop. Incorrect predictions are concentrated at higher entropy (ET) values, providing empirical justification for routing low-confidence cases through the backstop's deferral mechanism.}
\label{fig:uncertainty_distribution}
\end{figure}

\subsection{Language vs Image Information}
\subsubsection{Why Language Dominates} Two non\textendash exclusive explanations for the empirical discovery in Section~\ref{sec:res_ablation} are consistent with our results. First, generalist LLMs are trained on enormous corpora of medical text in which malignancy correlates with specific terms (\emph{spiculated}, \emph{lobulated}, \emph{ground\textendash glass}); a faithful language rendering of attributes therefore recovers a strong prior. Second, the structured attributes themselves carry most of the diagnostic signal in LIDC\textendash IDRI, and language is simply the most efficient packaging the LLM can consume. The synthetic\textendash corruption controls in Section~\ref{sec:res_ablation} are designed to disambiguate these two: if (a) negation collapses AUC and (b) paraphrase preserves it, the LLM is genuinely consuming semantic content, not just exploiting trigger terms.


\subsubsection{Why Image\textendash Statistics Inputs Fail} Low\textendash level histogram and GLCM statistics are not a representation that generalist LLMs were trained to interpret in a clinical sense. The result is consistent with a broader theme in the literature: GPT\textendash 4V and similar models do well on modality and anatomy recognition but poorly on lesion\textendash level pathology identification \cite{brin2025assessing}. Our paper does not contradict the value of vision\textendash language pipelines that align CT features to language \cite{shaukat2024lung, zhuang2025vision}; it shows instead that, when structured language input is available, generalist LLMs achieve much of the benefit without the alignment training.

\subsection{What the Theory Adds} The two theorems in Section~\ref{sec:theory} are not deep mathematical contributions; they are honest applications of classical tools to the specific framework. Their value is operational. Theorem~\ref{thm:triage} converts vague claims (``\method{} seems to work'') into auditable statements (``the combined error is at most $25.6\%$ with probability $0.95$'') that a clinical safety officer can read and verify. Theorem~\ref{thm:oracle} is the more conceptually interesting result: it shows that calibration (which is observable through ECE) controls the gap to the Bayes\textendash optimal deferral classifier. The empirical results support this prediction: Platt calibration reduced ECE by approximately 50\%, increased LLM decision coverage from 65.5\% to 76.5\%. Without Theorem~\ref{thm:oracle}, the calibration step in \method{} would be a heuristic aimed at producing more reliable confidence estimates; with it, calibration becomes a directly measurable mechanism for improving routing decisions and reducing triage suboptimality.

\subsection{Limitations}
\label{sec:limitations}
\begin{enumerate}[leftmargin=*, itemsep=1pt, topsep=2pt]
\item LIDC\textendash IDRI label subjectivity: The malignancy labels in LIDC\textendash IDRI are based on radiologist assessments rather than pathology\textendash confirmed diagnoses. Future work should validate the proposed framework on pathology\textendash proven datasets to assess its performance against definitive clinical ground truth.

\item No prospective clinical validation: This study presents a methodological framework evaluated on public benchmarks. We do not claim prospective deployment performance and did not conduct a recruited reader study. While the LIDC\textendash IDRI radiologist annotations provide a benchmark human reference, prospective external validation, multi\textendash institution reader studies, workflow\textendash integrated evaluation, and local calibration audits remain necessary before clinical deployment.

\item Limited multimodal public datasets: Public datasets providing both CT images and expert textual descriptions for lung nodule malignancy assessment are scarce. Consequently, this study is limited to LIDC\textendash IDRI, where structured textual descriptions were generated.

\end{enumerate}

\subsection{Design Recommendations for AI Developers} Three concrete recommendations follow from the present results, each tied to a specific finding rather than to the framework as a whole.

\begin{tcolorbox}[
  enhanced, breakable,
  colback=conftriage-lightgreen!22,
  colframe=conftriage-teal,
  boxrule=0.6pt, arc=2.5pt,
  left=8pt, right=8pt, top=5pt, bottom=7pt,
  before skip=6pt, after skip=6pt,
  title={Design Recommendations for AI Developers},
  coltitle=white, colbacktitle=conftriage-teal,
  fonttitle=\bfseries\small, toptitle=2pt, bottomtitle=2pt
]
\rec{R1}{Use language input, not raw\textendash image statistics, when invoking a generalist LLM as a triage classifier.}{%
Our seven\textendash way ablation (Table~\ref{tab:main_results}) shows that across five vendors, low\textendash level image statistics deliver near\textendash chance discrimination, while a deterministic natural\textendash language rendering of standard radiology attributes routinely yields AUC $\geq 0.79$ and reaches AUC = 0.907 for the strongest model. The evidence does not support feeding raw histogram or GLCM statistics as a substitute for a language description.}

\recsep

\rec{R2}{When deploying, calibrate first, then threshold.}{%
Theorem~\ref{thm:oracle} ties \method{}'s excess risk to the $L^1$ calibration error of the LLM probability. In practice this means: hold out a calibration fold, fit Platt scaling on the verbalized logit, then sweep the deferral threshold under a sensitivity\textendash floor constraint (we use $\Sens \geq 0.95$). Skipping the calibration step turns the deferral threshold into a heuristic with no formal guarantee.}

\recsep

\rec{R3}{Pair the LLM with a backstop that itself emits a calibrated rejection signal.}{%
The Certain\textendash Net backstop's MC\textendash dropout uncertainty separates correct from incorrect predictions; this is the empirical precondition that Theorem~\ref{thm:triage}'s certificate relies on. A backstop without an internal calibration audit (e.g., a single deterministic CNN with no uncertainty estimate) does not provide the same operational guarantee.}
\end{tcolorbox}

\subsection{Broader Impact} A triage system that defers hard cases to a specialist DL backstop is operationally appealing because it makes the failure mode explicit (low confidence $\to$ deferral) rather than hidden. This contrasts with a single\textendash model classifier whose miscalibration is invisible at use time. We caution that linguistic outputs from LLMs can be persuasive even when wrong, and any clinical deployment must include a verbalized\textendash confidence audit and a deferral path back to a human reader.


\section{Conclusion}
\label{sec:conclusion}

We introduced \method{}, a confidence-calibrated triage framework that combines a generalist LLM with a selective specialist DL backstop for pulmonary nodule malignancy prediction. \method{} achieved an F1 score of 0.882 while resolving 76.5\% of cases through zero-shot LLM inference alone. These findings suggest that structured radiological information expressed in natural language can effectively support LLM-based clinical triage. Ultimately, \method{} demonstrates that confidence-aware selective referral provides an effective mechanism for leveraging the complementary strengths of generalist LLMs and specialist AI models in clinical decision support.

\bibliographystyle{IEEEtran}
\bibliography{references}


\appendices

\section{Example Prompts Used in Prompt Ablation}
\label{app:prompt_ablation_examples}
We provide the prompt for each of the five prompt-template types (T1--T5) used
in the prompt-ablation study. In all cases, the same input description placeholder
\texttt{[NODULE\_DESCRIPTION]} was used.

\begin{promptbox}{Shared suffix appended in all templates (T1--T5)}
Nodule description:
[NODULE_DESCRIPTION]
Question: Is this nodule benign or malignant?
Return strict JSON only with keys: label, probability.
label must be 'Y' for malignant or 'N' for benign.
probability must be a float in [0, 1] for malignant probability.
\end{promptbox}

\begin{promptbox}{T1: (zero\_shot\_direct) prefix}
You are evaluating a pulmonary nodule.
Give a direct binary malignancy decision.
\end{promptbox}

\begin{promptbox}{T2: (thoracic\_radiologist\_persona) prefix}
You are a board-certified thoracic radiologist with expertise
in lung nodule risk assessment.
Use a clinically grounded interpretation of the nodule description.
\end{promptbox}

\begin{promptbox}{T3: (chain\_of\_thought) prefix}
Reason through the malignancy evidence step by step internally.
Do not reveal chain-of-thought. Output only the final JSON answer.
\end{promptbox}

\begin{promptbox}{T4: (three\_shot\_dcalib) prefix (default \texttt{n\_shots=3})}
Use these class-balanced in-context examples from Dcalib before answering.
Calibration examples (Dcalib):
Example 1:
Description: [DCALIB_EXAMPLE_1_TEXT]
Answer: {"label": "N", "probability": 0.10}
Example 2:
Description: [DCALIB_EXAMPLE_2_TEXT]
Answer: {"label": "Y", "probability": 0.90}
Example 3:
Description: [DCALIB_EXAMPLE_3_TEXT]
Answer: {"label": "N", "probability": 0.10}
\end{promptbox}

\begin{promptbox}{T5: (fleischner\_2017\_preamble) prefix}
Fleischner Society 2017 style preamble:
- Consider suspicious morphology and contextual risk from the reported descriptors.
- Prioritize features suggestive of malignancy such as irregular or spiculated patterns.
- Use conservative probabilistic judgment when descriptors are mixed.
\end{promptbox}
\end{document}